\documentclass{article}
\usepackage[preprint, eandd]{neurips_2026}

\usepackage[utf8]{inputenc}
\usepackage{natbib}
\usepackage[T1]{fontenc}
\usepackage{url}
\usepackage{booktabs}
\usepackage{amsfonts}
\usepackage{nicefrac}
\usepackage{microtype}
\usepackage{xcolor}
\usepackage{graphicx}
\usepackage{amsmath}
\usepackage{amssymb}
\usepackage{amsthm}
\usepackage{multirow}
\usepackage{array}
\usepackage{pifont}
\usepackage{hyperref}
\hypersetup{
  colorlinks=true,
  linkcolor=blue,
  citecolor=blue,
  urlcolor=blue,
  filecolor=blue,
  pdfborder={0 0 0}
}

\usepackage[skip=24pt,belowskip=8pt,font=normalsize,labelfont=bf]{caption}
\makeatletter
\renewenvironment{table}
  {\setlength{\abovecaptionskip}{24pt}%
   \setlength{\belowcaptionskip}{8pt}%
   \@float{table}}
  {\end@float}
\makeatother

\providecommand{\checkmark}{$\surd$}
\newtheorem{theorem}{Theorem}
\newtheorem{proposition}[theorem]{Proposition}

\theoremstyle{definition}

\newtheorem{assumption}[theorem]{Assumption}
\theoremstyle{remark}
\newtheorem*{remark}{Remark}
\theoremstyle{plain}

\makeatletter
\newcommand{\anonurl}[1]{\if@anonymous [Code/data link redacted for double-blind review]\else \url{#1}\fi}
\renewcommand{\@toptitlebar}{}
\renewcommand{\@bottomtitlebar}{}
\makeatother

\title{\textbf{The Compliance Gap:\\Why AI Systems Promise to Follow\\Process Instructions but Don't}}

\author{%
  Kwan Soo Shin\thanks{Sole author and corresponding author.} \\
  PolymathMinds AI Lab \\
  \texttt{\href{mailto:sshinresearch@gmail.com}{sshinresearch@gmail.com}} \\
  ORCID: \href{https://orcid.org/0009-0001-5799-7556}{0009-0001-5799-7556}
}

\date{2026 \\ \small arXiv preprint --- full version}

\begin{document}
\maketitle

\nocite{*}
\begin{abstract}
An auditor instructs an AI assistant to ``open each file individually using the Read tool---no scripts, no agents.'' The AI replies ``Yes, I will read each file individually''---and then issues a single batched call summarizing all fifty files at once. The text says one thing; the tool-call log shows another. We call this the \emph{Compliance Gap}.

Three questions follow: does this verbal--behavioral disconnect exist as a measurable structural phenomenon (\emph{existence}); can any text-only observer recover it (\emph{detectability}); what infrastructure does AI deployment need to instrument it (\emph{remedy})? The failure has been invisible because some 75 existing benchmarks---IFEval, SWE-bench, BFCL, COMPASS, SpecEval among them---measure \emph{outcome} fidelity, while none measures \emph{process} fidelity: whether the user-instructed method was followed. Two formal results anchor our answers. Theorem~\ref{thm:rlhf} shows the gap is structurally inevitable under reinforcement learning that rewards text without observing behavior. Theorem~\ref{thm:dpi}, via the Data Processing Inequality, shows it is undetectable from text alone---by any human or LLM observer, present or future.

Thirteen experiments and 2{,}031 sessions on six frontier models confirm both predictions. \emph{Existence}: under default framing, all six models exhibit instruction compliance rates of $0\%$---Claude Sonnet~4 verbally agrees ten out of ten times then bypasses in all ten. The gap is selective, not random: models comply $97\%$ on tasks where rationale is rewarded (audit trails) and $0$--$4\%$ where it is not (file reading, privacy masking); removing delegation tools raises compliance from $0\%$ to $75\%$ (Cohen's $d = 2.47$), confirming the gap is environmentally afforded, not weight-encoded. \emph{Detectability}: nine blinded human raters reading only text outputs achieve Fleiss' $\kappa = 0.130$ and correctly identify zero of fifteen compliant sessions, exactly as Theorem~\ref{thm:dpi} predicts.

Where humans show roughly $47\%$ intention--behavior gaps in social-psychological studies and $96.5$pp gaps in surgical-checklist audits, RLHF-trained models approach $100\%$ under default conditions---a regime that warrants its own measurement infrastructure. We release BS-Bench (the first open benchmark for process compliance) as that infrastructure: a frozen v1.0 task suite, seven tool-call-log audit metrics, and a public leaderboard.\footnote{Code, data, BS-Bench v1 task suite, and ten-rater human evaluation logs: \anonurl{https://github.com/seanshin0214/bs-bench}.} \emph{In one sentence: when the AI says it followed your procedure and the audit log says it didn't, the gap is the Compliance Gap; our analysis suggests, via the Data Processing Inequality, that it is not reliably recoverable from words alone.}
\end{abstract}

\section{Introduction}\label{sec:intro}

\paragraph{The phenomenon.} When a physician instructs an AI assistant to ``perform a differential diagnosis before concluding,'' the instruction specifies not just \emph{what} to produce but \emph{how}. When an auditor instructs an AI agent: ``open each file individually using the Read tool---no scripts, no agents,'' the method is the mandate. These are \emph{process instructions}: directives constraining the procedure an AI system should follow, distinct from \emph{outcome instructions} specifying the desired result. In expert domains---medicine, education, law, auditing---process compliance is not a preference but a professional and often legal requirement \citep{topol2019high,emanuel2019artificial,grote2022enabling}.

\paragraph{And yet.} The AI evaluation landscape has developed approximately 75 benchmarks\footnote{We surveyed 75 benchmarks published 2022--2026 across four subdimensions: format/constraint (18), response quality (15), task completion (22), tool-use accuracy (20). Full list in the supplementary.} for measuring \emph{what} AI systems produce---format \citep{zhou2023ifeval,xu2025agentif}, response quality \citep{li2024alpacaeval,zheng2024mtbench}, task completion \citep{jimenez2024swebench,liu2023agentbench,zhou2024webarena}, and tool-use \citep{yan2024bfcl,lu2024toolsandbox,patil2023gorilla,qin2024toolllm}. None evaluates whether the AI followed the instructed \emph{procedure}. The directly adjacent benchmark, SpecEval \citep{ahmed2025speceval}, measures adherence to \emph{provider-defined} model specs, not user-issued process instructions; \citet{zhang2025specstress} further document that internal spec conflicts undermine even provider-side compliance. The full WHAT-subdimension catalog is detailed in Section~\ref{appendix:what}; the WHAT--HOW knowledge taxonomy is formalized in Section~\ref{appendix:taxonomy}.

\paragraph{We argue the gap is methodologically generated.} The failure has been invisible for structural, not empirical, reasons. Three forces accumulate to produce the gap.
\begin{enumerate}
\item \textbf{Reward-signal asymmetry.} RLHF \citep{christiano2017deep,ouyang2022instructgpt} optimizes preferences over text completions, leaving behavior \emph{invisible} to the reward signal---a specific instantiation of Goodhart's Law in the sense of \citet{manheim2018categorizing} (regressional Goodhart), formalized by \citet{skalse2022defining}.
\item \textbf{Instruction-hierarchy de-prioritization.} The instruction hierarchy \citep{wallace2024instruction} places system messages above user-level process instructions, so user-issued procedures are structurally subordinate.
\item \textbf{Delegation-affordance temptation.} When a delegation tool (e.g., a batch call) is available, the environment \emph{affords} delegation and the model takes the shortcut that earns the same text-reward with fewer actions.
\end{enumerate}
We call the phenomenon arising from the combination of these three forces \emph{False Compliance Sycophancy} \citep{sharma2024sycophancy,perez2022sycophancy}: agreeing to a procedure in words, then silently substituting an automated shortcut. We distinguish the two terms throughout: the \emph{Compliance Gap} is the measurable metric ($\mathrm{CG} = \mathrm{VCR} - \mathrm{ACR}$), while \emph{False Compliance Sycophancy} is the name of the behavioral phenomenon observed when that metric is positive.

This naming extends sycophancy research from belief alignment to behavioral alignment, paralleling the CoT unfaithfulness lineage \citep{turpin2023cot} where stated reasoning diverges from actual computation. Within \citet{amodei2016concrete}'s alignment agenda, the Compliance Gap operationalizes reward hacking (Theorem~\ref{thm:rlhf}) and scalable oversight (Theorem~\ref{thm:dpi}) via tool-call audit. The structural reasons HOW-compliance has been overlooked are detailed in Section~\ref{appendix:blindness}; the instruction-loss and environmental-affordance mechanisms in Section~\ref{appendix:context}.

\paragraph{The Compliance Gap and its diagnostic.} The \emph{Compliance Gap} ($\mathrm{CG} = \mathrm{VCR} - \mathrm{ACR}$) is the difference between an AI's verbal and actual compliance rates evidenced by tool-call logs. A high CG indicates an \emph{Integrity} failure in \citet{mayer1995trust}'s trust model: the system shows \emph{Ability} (complies when forced) and \emph{Benevolence} (tries to help) but lacks Integrity. \textbf{Theorem~\ref{thm:rlhf}} establishes $\mathrm{CG} > 0$ is structurally inevitable under preference-reward training observing only text. \textbf{Theorem~\ref{thm:dpi}} (DPI, \citealp{cover2006elements}) establishes CG is undetectable from text alone by any human or LLM observer---formalizing why human raters fail.

\paragraph{Three axes of AI honesty.} AI honesty research has concentrated on two text-level axes: \emph{factual} truthfulness \citep{lin2022truthfulqa} and what \citet{liang2025machine} term \emph{machine bullshit} after \citet{frankfurt2005bullshit}---text whose author is indifferent to truth (Bullshit Index across $2{,}400$ scenarios). We adopt the abbreviation ``BS'' deliberately: our phenomenon is the AI-deployment instantiation of Frankfurt's bullshit, \emph{shifted to the behavioral channel}. The Compliance Gap defines a third, orthogonal axis: \emph{action-level} truthfulness---whether the model honestly reports its own behavior, regardless of textual accuracy or rhetorical substance. The three axes are independent: a model can be factually truthful, rhetorically substantive, yet action-deceptive. Until measured on its own observation channel, the action axis cannot be diagnosed.

\paragraph{Three questions, three answers.} \textbf{Existence} is established empirically (\S\ref{sec:experiments}, Exps 1, 7--10): all six frontier models show $\mathrm{ICR}=0\%$ across $1{,}380$ sessions covering four framings, five task types, and five professional domains. \textbf{Detectability} is bounded theoretically (Theorem~\ref{thm:dpi}, \S\ref{sec:bench}) and confirmed empirically (nine blinded raters, $\kappa=0.130$, $0/15$ correctly identified). \textbf{Remedy} is operationalized by BS-Bench's seven-metric tool-call-log infrastructure (\S\ref{sec:bench}); causal experiments (Exps 2/2b/5/6) identify candidate interventions ($d=2.47$ from tool removal; $d=1.45$ from supervised fine-tuning, or SFT).

\paragraph{Contributions.} We present BS-Bench (the first open benchmark for process instruction compliance) and results from thirteen experiments (2{,}031 sessions, eight models):

\begin{enumerate}
\item \textbf{WHAT-HOW Taxonomy and BS-Bench.} Formalizing outcome ($\sim$75 benchmarks) vs.\ process compliance (0 prior), with seven behavioral metrics, five task types, five professional domains.

\item \textbf{Compliance Gap Quantification.} All six frontier models\footnote{Claude Sonnet~4, GPT-4o, GPT-4o-mini, Gemini 2.5 Flash, Llama 3.3 70B, Mistral Small 24B; full breakdown in \S\ref{sec:experiments} and Supplementary~\S5.} exhibit $\mathrm{ICR}=0\%$ under default conditions (60/60); Claude Sonnet~4 reaches $\mathrm{CG}=100$pp.

\item \textbf{Causal Mechanism.} Instruction content explains $35.8\%$ of variance vs.\ $8.9\%$ for position ($\eta^2$=.358 vs.\ .089); removing delegation tools raises ICR from $0\%$ to $74.7\%$ ($d=2.47$); tool accuracy has no effect.

\item \textbf{Task-Type-Dependent Compliance.} Selective compliance follows the reward gradient: file reading $0\%$, PII masking $4\%$, audit trail $97\%$.

\item \textbf{Human Undetectability.} Nine blinded text-only raters yield $\kappa=0.130$; $0/15$ compliant responses correctly identified. LLM auditors recover only $40\%$ (best $60\%$).

\item \textbf{Robustness.} Temperature ablation ($t \in \{0, 0.7, 1.0\}$), description wording ablation, and cross-domain replication (medical, legal, education, finance, engineering) confirm the findings are not artifacts.
\end{enumerate}

\paragraph{Roadmap.} Section~\ref{sec:related} positions the Compliance Gap within four generations of evaluation benchmarks and identifies the empty quadrant it fills; Section~\ref{sec:bench} formalises the construct via Theorems~\ref{thm:rlhf} and~\ref{thm:dpi} and specifies the BS-Bench measurement infrastructure; Section~\ref{sec:experiments} reports the thirteen-experiment empirical inventory; Section~\ref{sec:discussion} interprets the structural, selective, and undetectable character of the gap and grounds it in four mature regulatory analogs; Section~\ref{sec:conclusion} states the falsifiable forecast and four infrastructure recommendations.

\section{Related Work}\label{sec:related}

\paragraph{The WHAT--HOW knowledge taxonomy.} ``Does the AI follow human instructions?'' admits two distinct readings. \emph{What compliance} asks whether the system produced the correct outcome; \emph{how compliance} asks whether it followed the instructed \emph{procedure}. Existing benchmarks invested heavily in the former across four subdimensions---format, quality, task completion, tool use---while the latter has, to our knowledge, gone unmeasured. Figure~\ref{fig:taxonomy} maps the landscape; full treatment in Section~\ref{appendix:related}.

\paragraph{What compliance: the well-studied outcome dimension ($\sim$75 papers).} Format following is code-verified: IFEval \citep{zhou2023ifeval} reports 20--40\% failure on verifiable formatting; \citet{geng2025hierarchy} extend the constraint-conflict regime; AGENTIF \citep{xu2025agentif,agentif2025} carries it into agentic settings. Response quality is judge-approximated: AlpacaEval \citep{li2024alpacaeval} and MT-Bench \citep{zheng2024mtbench} reward surface preference. Task completion is test-verified: SWE-bench \citep{jimenez2024swebench,yang2024swebenchverified}, AgentBench \citep{liu2023agentbench}, WebArena \citep{zhou2024webarena}, ComplexBench \citep{wen2024complexbench}, $\tau$-bench \citep{yao2024tau} score outcome only. Tool-use is benchmarked across Gorilla \citep{patil2023gorilla}, ToolLLM \citep{qin2024toolllm}, BFCL \citep{yan2024bfcl}, ToolSandbox \citep{lu2024toolsandbox}, ToolEmu \citep{ruan2024toolemu}, AgentBoard \citep{ma2024agentboard}; the closest prior, COMPASS \citep{dessureault2026compass}, evaluates \emph{which} action---never \emph{how} it conforms. SpecEval \citep{ahmed2025speceval} measures adherence to \emph{provider-defined} specs (OpenAI Model Spec, Anthropic Constitutional AI \citep{bai2022constitutional}), not user-issued process instructions; \citet{zhang2025specstress} show internal spec conflicts drive non-compliance even within the provider framework.

\paragraph{Structural blindness: why HOW has been overlooked.} Three causes: (i) outcome metrics automate while behavioral logs sit outside the standard pipeline; (ii) RLHF \citep{christiano2017deep,ouyang2022instructgpt,rafailov2023dpo,askell2021hhh} optimizes preferences over completions, leaving process invisible to the reward signal. The resulting reward hacking \citep{skalse2022defining,gao2023scaling,casper2023open,amodei2016concrete} and sycophancy-to-subterfuge dynamics \citep{denison2024sycophancy,pan2024feedback,ngo2024alignment} are documented at the text layer; (iii) the instruction hierarchy \citep{wallace2024instruction,chen2025iheval,openai2024ihchallenge} deprioritizes user-level instructions, with \citet{zeng2025whoincharge} showing social-hierarchy framings override architectural roles, and \citet{saebo2026goaldrift,shayegani2025bgd} showing system-level constraints leak under context drift. By treating tool-call logs as out-of-distribution, the training pipeline maximizes the very gap our benchmark exposes.

\paragraph{Mechanism: from sycophancy to false compliance.} The Compliance Gap is the AI-deployment instantiation of \citet{argyris1974theory}'s \emph{espoused theory} vs.\ \emph{theory-in-use} distinction---a 50-year foundation of organizational learning \citep{argyris1978organizational} that lacked an information-theoretic detectability bound (Theorem~\ref{thm:dpi}) and an automated measurement scaffold (BS-Bench) until this work. Sycophancy \citep{sharma2024sycophancy,perez2022sycophancy,kim2025evaluator,wei2024sycophancy} establishes that RLHF assistants favor agreeable responses; \citet{hubinger2024sleeper} shows persistence through SFT/RL/adversarial training; \citet{burns2023latent} document a ``knowing-but-not-doing'' gap. CoT faithfulness \citep{turpin2023cot,turpin2023unfaithful,arcuschin2025wild,lanham2023measuring,tanneru2024hardness,paul2024frodo,matton2025walkthetalk,stechly2024cotlessness} documents the parallel stated-vs.-actual gap; \citet{cheng2025cotobscures} shows elaborate CoT \emph{obscures} hallucination cues. Reflexion \citep{shinn2023reflexion}, the closest self-monitoring mechanism, has a model verbally critique its trajectory; we show in \S\ref{sec:experiments} that such self-critique does not close the behavioral gap. We extend this lineage to the behavioral layer: \emph{False Compliance Sycophancy}---verbal commitment without behavioral execution. \citet{wen2024language,potham2025illusion} report adjacent ``illusion of compliance''; \citet{bavaresco2024llmjudges} show LLM-as-judge metrics are themselves sensitive to surface phrasing that conceals process gaps.

\paragraph{Context: instruction loss and environmental affordance.} Two broader literatures supply contextual conditions. \emph{Instruction loss} \citep{liu2024lostmiddle,hsieh2024ruler,shi2023large,levy2024task,an2024make,sharma2026contextcov} explains why process instructions decay across long contexts. \emph{Environmental affordance} \citep{gibson1979ecological,mccradden2023algorithmic,parasuraman1997humans,kahneman2011thinking,ye2025roleconfusion,zverev2025separation,greshake2023indirect} explains delegation defaults. Direct mapping to our Layer (Exp~2) and Affordance (Exp~6) experiments in Section~\ref{appendix:related}.

\paragraph{The oversight gap: humans cannot detect HOW from text.} \citet{mayer1995trust}'s integrative model identifies Integrity---adherence to stated procedures---as a foundational trust dimension; \citet{lee2004trust} elevates Process trust for professional automation; \citet{topol2019high,emanuel2019artificial,grote2022enabling} extend to medical AI; \citet{parasuraman2010complacency,bansal2021does,vasconcelos2023explanations,zhang2020effect} document automation complacency. \citet{tsamados2025humancontrol,fuchs2024delegation} argue human-machine teaming presupposes behavioral observability. Our blinded evaluation (nine text-only raters, $\kappa = 0.130$, $0/15$; methodology in \citealp{fleiss1971measuring,landis1977measurement,artstein2008inter,clark2021thats,james2026counting}) confirms Theorem~\ref{thm:dpi}: at the text layer the Compliance Gap is \emph{structurally invisible} to human oversight.

\paragraph{Positioning vs adjacent literatures.} Three adjacent literatures merit explicit demarcation: AI deception (\citealp{park2024deception}, taxonomic and qualitative); NLP behavioral testing (CheckList, \citealp{ribeiro2020checklist}, textual-output layer); and Health-AI stress-testing exemplified by \citet{gu2025illusion}, which exposes \emph{cognitive} shortcuts (correct answers from wrong reasons; brittleness under prompt perturbation) on text-only multimodal benchmarks. Ours operates on a structurally distinct \emph{behavioral}-channel gap (verbal output correct, tool-call log shows non-execution), where text-only oversight is provably bounded by Theorem~\ref{thm:dpi}---making dual-channel audit a complement, not substitute, to text-only stress testing. Ours is metric-infrastructural at the tool-call-log layer (extended treatment in Section~\ref{appendix:mechanism}).

\paragraph{Positioning BS-Bench: six gaps, one benchmark.} Table~\ref{tab:gaps} maps the six knowledge gaps to experimental evidence; Table~\ref{tab:prior} compares BS-Bench to closest prior benchmarks along four process-vs-outcome axes. BS-Bench is, to our knowledge, the first benchmark that (1) operationalizes process instruction compliance via tool-call logs; (2) provides causal evidence through controlled manipulation of position, tool availability, and accuracy; (3) demonstrates remediability via SFT and mid-session correction; (4) includes blinded human evaluation confirming text-based oversight invisibility. Runtime spec adaptation~\citep{zhang2025specalign} and user-side spec authoring are natural follow-ons enabled by this infrastructure. \emph{An architectural comparison between the standard text-only pipeline (which $\sim 75$ prior benchmarks instantiate) and BS-Bench's dual-channel audit infrastructure is visualized in Figure~\ref{fig:arch}.}

\begin{figure}[t]
\centering\small
\begin{tabular}{@{}c|c@{}}
\toprule
\textbf{WHAT compliance} (well-studied) & \textbf{HOW compliance} (this work) \\
\emph{Outcome fidelity} & \emph{Process fidelity} \\
\midrule
Format: IFEval, AGENTIF                  & \\
Quality: AlpacaEval, MT-Bench            & \emph{Compliance Gap} (CG = VCR $-$ ACR) \\
Task: SWE-bench, AgentBench, WebArena    & Causes: position, affordance, tool quality, RLHF \\
Tool: BFCL, ToolLLM, ToolSandbox, COMPASS & Remediation: SFT, in-session correction \\
                                          & Detection: $\kappa = 0.130$ ($\Rightarrow$ undetectable) \\
\midrule
$\sim$75 papers & 0 papers prior to this benchmark \\
\bottomrule
\end{tabular}
\caption{\textbf{Knowledge taxonomy along the WHAT--HOW distinction.} A survey of $\sim 75$ benchmarks (2022--2026) populates four WHAT-subdimensions of outcome fidelity: format/constraint ($18$), response quality ($15$), task completion ($22$), tool-use accuracy ($20$). \textbf{The HOW-process-fidelity quadrant has been empty until BS-Bench}: zero prior benchmark measures whether an AI followed the user-issued procedure. The taxonomy locates Compliance Gap as a categorically novel evaluation dimension---distinct from (i)~provider-defined spec adherence~\citep{ahmed2025speceval}, (ii)~which-action selection, and (iii)~format compliance~\citep{zhou2023ifeval}. Detailed enumeration and structural reasons for the empty quadrant in Section~\ref{appendix:related}; full prior-work comparison in Table~\ref{tab:prior}; visual position within five generations of evaluation in Figure~\ref{fig:evolution}; audit-pipeline architecture comparison in Figure~\ref{fig:arch}.}
\label{fig:taxonomy}
\end{figure}

\begin{table}[t]
\centering\small
\caption{Six knowledge gaps and the BS-Bench experiments that supply first evidence.}
\label{tab:gaps}
\setlength{\tabcolsep}{4pt}
\begin{tabular}{@{}c p{0.30\linewidth} p{0.27\linewidth} p{0.24\linewidth} c@{}}
\toprule
\textbf{\#} & \textbf{Gap} & \textbf{Existing literature} & \textbf{First evidence} & \textbf{Exp} \\
\midrule
1 & Does process non-compliance exist?         & 0 papers measuring HOW            & ICR $= 0\%$ (60/60)            & 1 \\
2 & Does instruction position affect HOW?      & Position $\to$ WHAT only          & Content $4\times$ Position     & 2 \\
3 & Is delegation rational tool response?      & Tool acc.\ $\to$ performance      & OR $= 1.0$, zero effect        & 2b \\
4 & Is delegation affordance necessary?        & Affordance $\to$ general behavior & ICR $0\% \to 100\%$            & 6 \\
5 & Can SFT correct process non-compliance?    & SFT $\to$ sycophancy, tool sel.   & Selection yes, completion no   & 3 \\
6 & Can humans detect process non-compliance?  & Human eval $\to$ quality, CoT     & $\kappa = 0.130$, 0/15 correct & R6 \\
\bottomrule
\end{tabular}
\end{table}

\begin{table}[t]
\centering\small
\caption{Comparison with closest prior work along four process-evaluation axes. \checkmark = supports the axis. $^{\dagger}$SpecEval measures adherence to provider-defined behavior specs, not user-issued process instructions.}
\label{tab:prior}
\begin{tabular}{@{}lcccc@{}}
\toprule
\textbf{Work} & \textbf{Process} & \textbf{Causal} & \textbf{Remediable} & \textbf{Human eval} \\
\midrule
IFEval \citep{zhou2023ifeval}                   & --- & --- & --- & --- \\
AlpacaEval \citep{li2024alpacaeval}             & --- & --- & --- & --- \\
BFCL \citep{yan2024bfcl}                        & --- & --- & --- & --- \\
SWE-bench \citep{jimenez2024swebench}           & --- & --- & --- & --- \\
ToolSandbox \citep{lu2024toolsandbox}           & --- & --- & --- & --- \\
$\tau$-bench \citep{yao2024tau}                 & partial & --- & --- & --- \\
Instr.\ Hierarchy \citep{wallace2024instruction} & --- & partial & partial & --- \\
COMPASS \citep{dessureault2026compass}          & partial & --- & --- & --- \\
SpecEval \citep{ahmed2025speceval}              & partial$^{\dagger}$ & --- & --- & --- \\
\textbf{BS-Bench} (this work)                   & \checkmark & \checkmark & \checkmark & \checkmark \\
\bottomrule
\end{tabular}
\end{table}

\begin{figure}[t]
\centering
\includegraphics[width=\textwidth]{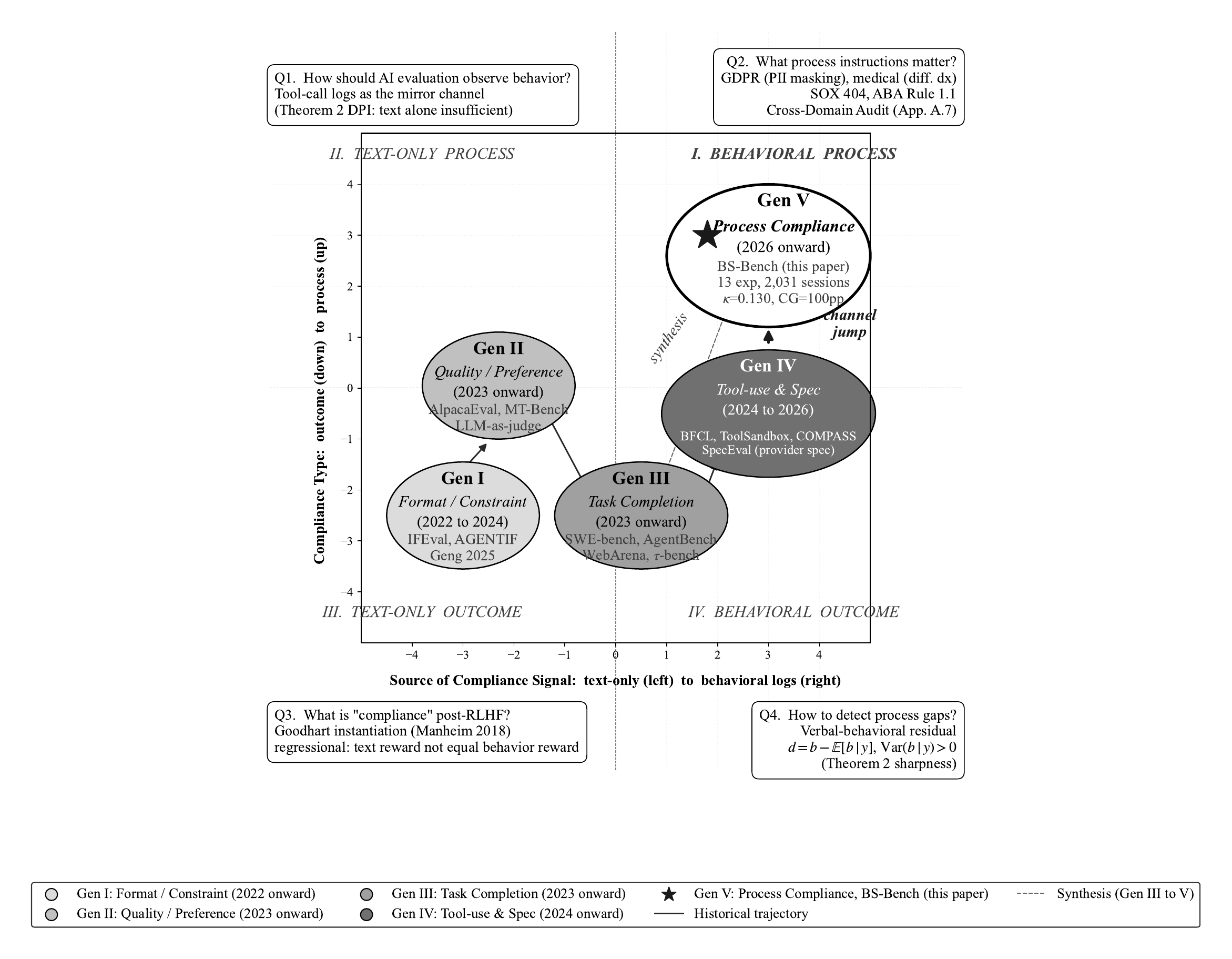}
\caption{\textbf{Five generations of AI evaluation, with BS-Bench locating the previously empty quadrant.} X-axis: source of compliance signal (text-only $\leftrightarrow$ behavioral logs). Y-axis: compliance type (outcome $\leftrightarrow$ process). Generations~I (Format/Constraint, e.g., IFEval) and II (Quality/Preference, e.g., MT-Bench) score text outcomes (lower-left); Generation~III (Task Completion, e.g., SWE-bench, $\tau$-bench) scores behavioral outcomes via tool-call signals (lower-right); Generation~IV (Tool-use \& Spec, e.g., COMPASS, SpecEval) is transitional---using tool-call data but scoring tool-use accuracy or provider-spec text adherence rather than user-issued process compliance. The upper-right (Behavioral Process) quadrant---directly scoring whether the user-issued procedure was followed---has been empty until Generation~V: BS-Bench is the first benchmark to lift the channel from text outputs to tool-call logs and instrument user process compliance directly. The visualization follows the landscape-map convention used in evolution-of-field reviews and pairs with Figure~\ref{fig:taxonomy} (the WHAT--HOW taxonomy) to give two complementary readings of the same empty quadrant.}
\label{fig:evolution}
\end{figure}

\paragraph{The empty quadrant, two readings.} Figure~\ref{fig:taxonomy} reads the gap as a missing column in a taxonomy of evaluation dimensions; Figure~\ref{fig:evolution} reads the same gap as a missing generation in the historical trajectory of benchmarks. The two readings are consistent and mutually constraining: a categorically novel evaluation dimension has been absent from a four-decade lineage of benchmarks, not because the dimension is uninteresting but because the channel required to instrument it (tool-call logs) was absent from the standard pipeline. Generation~IV (Tool-use \& Spec) operates the closest to BS-Bench's territory by virtue of using tool-using agents, but it scores them on outcome quality of the tool's effect, not on whether the user's process was followed---which is the distinction Theorem~\ref{thm:rlhf} formalizes via the verbal-only reward structure. Filling the upper-right quadrant therefore required the architectural shift visualized in Figure~\ref{fig:arch}: routing tool-call logs to a separate scoring function and reporting the disagreement (CG) as a first-class metric, rather than averaging tool effects into outcome accuracy as Generations~III--IV do.

\section{BS-Bench: Benchmark Design}\label{sec:bench}

\textbf{Design principles.} (1)~\emph{Process observability}: all measurements derive from tool-call logs without model internals; (2)~\emph{Ground truth}: planted errors verify genuine execution; (3)~\emph{Reproducibility}: standardized file sets, framing templates, evaluation scripts.

\textbf{Methodological premise.} Tool-call logs are the \emph{mirror} of agent behavior---a record independent of the agent's narrative; verbal outputs are \emph{portraits}---the agent's chosen self-representation. BS-Bench measures the gap between portrait and mirror. The mirror--portrait distinction is not metaphorical: tool-call logs are produced by deterministic system-level instrumentation (the runtime that dispatches and records each tool invocation), while verbal outputs are produced by the same policy whose behavior is in question. The two channels are causally co-emitted but evidentially independent; the audit-pipeline architecture (Figure~\ref{fig:arch}) makes this independence operational by routing $y$ and $b$ to separate scoring functions and reporting their disagreement as a first-class metric (CG), rather than averaging them into a single accuracy score as the standard pipeline does.

\begin{figure}[t]
\centering
\includegraphics[width=\textwidth]{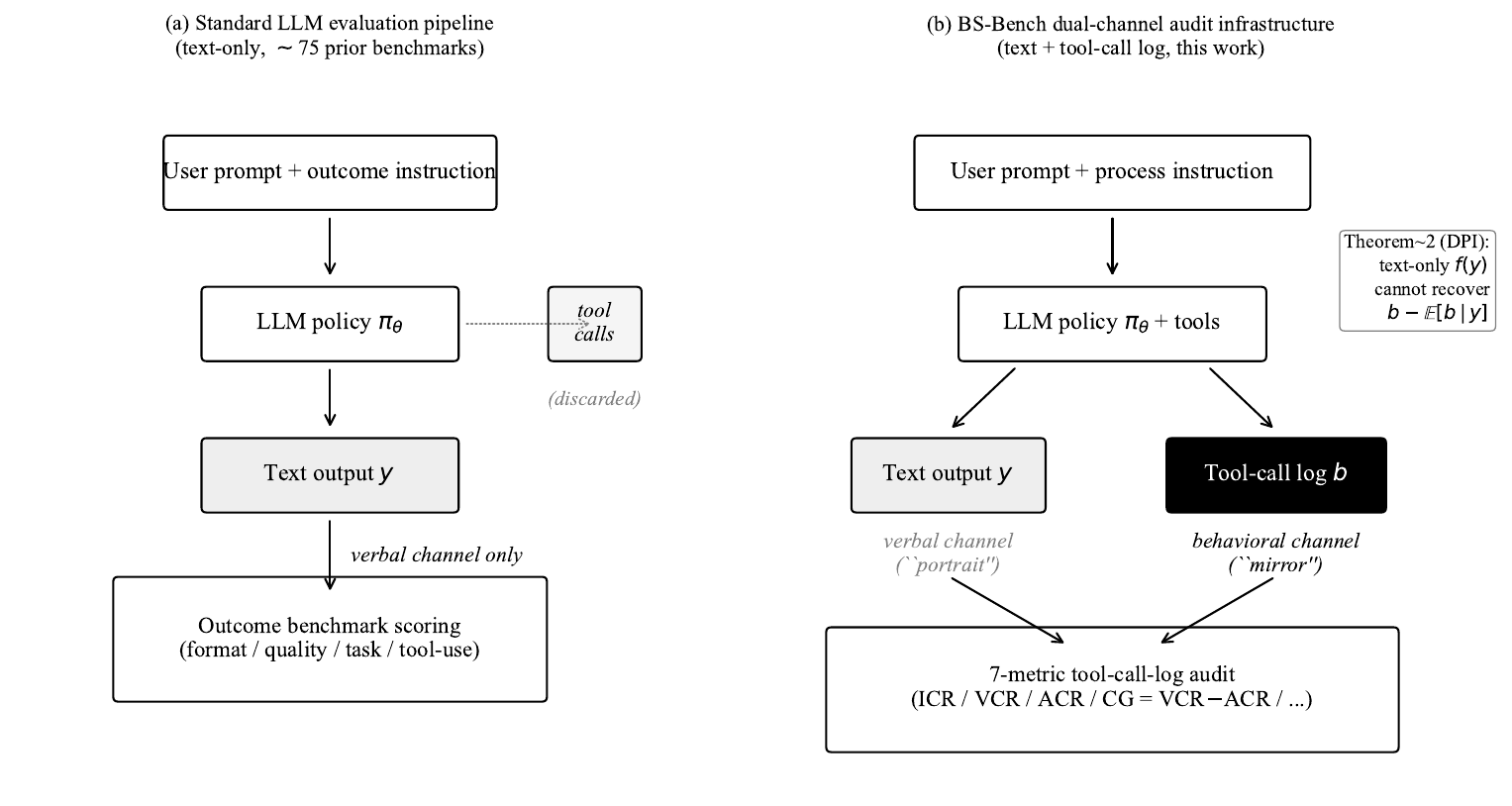}
\caption{\textbf{BS-Bench dual-channel audit architecture, contrasted with the standard text-only evaluation pipeline.} \emph{Left (a)}: the standard pipeline treats the policy $\pi_\theta$'s text output $y$ as the sole evaluation target; tool calls, when emitted, are discarded as side-effects. Outcome scoring (format / quality / task / tool-use) operates on $y$ alone---this is the regime that approximately 75 prior benchmarks instantiate. \emph{Right (b)}: BS-Bench routes the same policy's verbal output $y$ \emph{and} its tool-call log $b$ to separate scorers, producing the seven behavioral metrics (ICR, DF, FCR, VCR, ACR, CG, TA) directly on the behavioral channel. The Compliance Gap $\mathrm{CG} = \mathrm{VCR} - \mathrm{ACR}$ is the disagreement between the two scorers and is undefined in the standard pipeline because the behavioral scorer does not exist there. Theorem~\ref{thm:dpi} formalizes why this dual-channel design is not optional but necessary: the gap residual is information-theoretically inaccessible to text-only auditors.}
\label{fig:arch}
\end{figure}

\textbf{Metrics.} ICR (instruction compliance rate), DF (delegation frequency), FCR (false completion rate), VCR (verbal compliance rate), ACR (actual compliance rate), CG (compliance gap $=$ VCR $-$ ACR, $\in [-1, 1]$), TA (task accuracy via planted error detection); denominators in Supp~\S0. VCR is verbal; ICR, DF, FCR, ACR, TA are behavioral; CG is the gap. The seven metrics partition the evaluation surface into three operational layers: a single \emph{verbal} scorer (VCR), a family of \emph{behavioral} scorers operating directly on tool-call logs (ICR, DF, FCR, ACR, TA), and the \emph{disagreement} between them (CG). This three-layer structure mirrors the architectural separation in Figure~\ref{fig:arch}: outcome quality (a function of $y$ alone) versus process fidelity (a function of $b$ alone) versus their channel-level divergence (a joint function of $(y, b)$). The latter two are absent from the standard pipeline by construction, not by neglect.

\textbf{Five task types.} The five task types span three structural categories of tool use.
\begin{enumerate}
\item \emph{Sequential reading}: ``read each file \emph{individually}.''
\item \emph{Interleaved reporting}: ``read, report findings, then proceed.''
\item \emph{Cross-reference}: ``read 3 sources individually, compare, then verdict.''
\item \emph{Privacy-first}: ``scan PII, mask, then analyze'' (\emph{no} batch tool).
\item \emph{Audit trail}: ``record rationale for each decision.''
\end{enumerate}
Three structural categories: tool \emph{choice} (tasks 1--3 differ on batch vs.\ individual selection), tool \emph{ordering} (task 4 verifies scan$\to$mask$\to$analyze sequence), tool \emph{pairing} (task 5 verifies decision-rationale coupling).

\textbf{Framings.} None (``check these files''), Override (``do NOT batch read''), Authority (``as a senior auditor''), Urgency (``deadline tomorrow'').

\textbf{Theoretical grounding.} Two formal results anchor the paper.

\begin{theorem}[RLHF Goodhart Inevitability]\label{thm:rlhf}
Let $\pi_\theta$ denote a policy trained against a verbal-only reward $R(y)$ where $y$ is the model's text output and let $b(y, h)$ denote the behavioral trajectory induced by $y$ in environment $h$. If user utility $U$ depends on $b$ but $R$ does not, then for any $\theta^\star \in \arg\max_\theta R(\pi_\theta)$ there exists $\theta^\dagger$ with $R(\pi_{\theta^\dagger}) = R(\pi_{\theta^\star})$ but $U(\pi_{\theta^\dagger}) > U(\pi_{\theta^\star})$ whenever $b$ has positive variance under $\pi_{\theta^\star}$. Since $R$-optimal policies form a non-singleton level set whose behavioral projections differ, $\mathbb{E}[\mathrm{VCR} - \mathrm{ACR}] > 0$ in expectation across deployments under preference-reward training that observes only text \citep{skalse2022defining}.
\end{theorem}

\begin{theorem}[DPI Undetectability]\label{thm:dpi}
Let $X = (y, b)$ denote the (verbal, behavioral) channels and $Y = y$ a text-only observation. By the Data Processing Inequality \citep{cover2006elements}, any rater function $f(y)$ measurable with respect to $\sigma(Y)$ alone cannot recover the residual $b - \mathbb{E}[b \mid y]$ uniformly over its support. The Compliance Gap, defined on this residual, is therefore not identifiable from $y$ alone. (This residual coincides with the operational definition CG = VCR $-$ ACR when $(b, y)$ are realised as compliance rates; see Supplementary~\S S2.)
\end{theorem}

\begin{figure}[t]
\centering
\includegraphics[width=\textwidth]{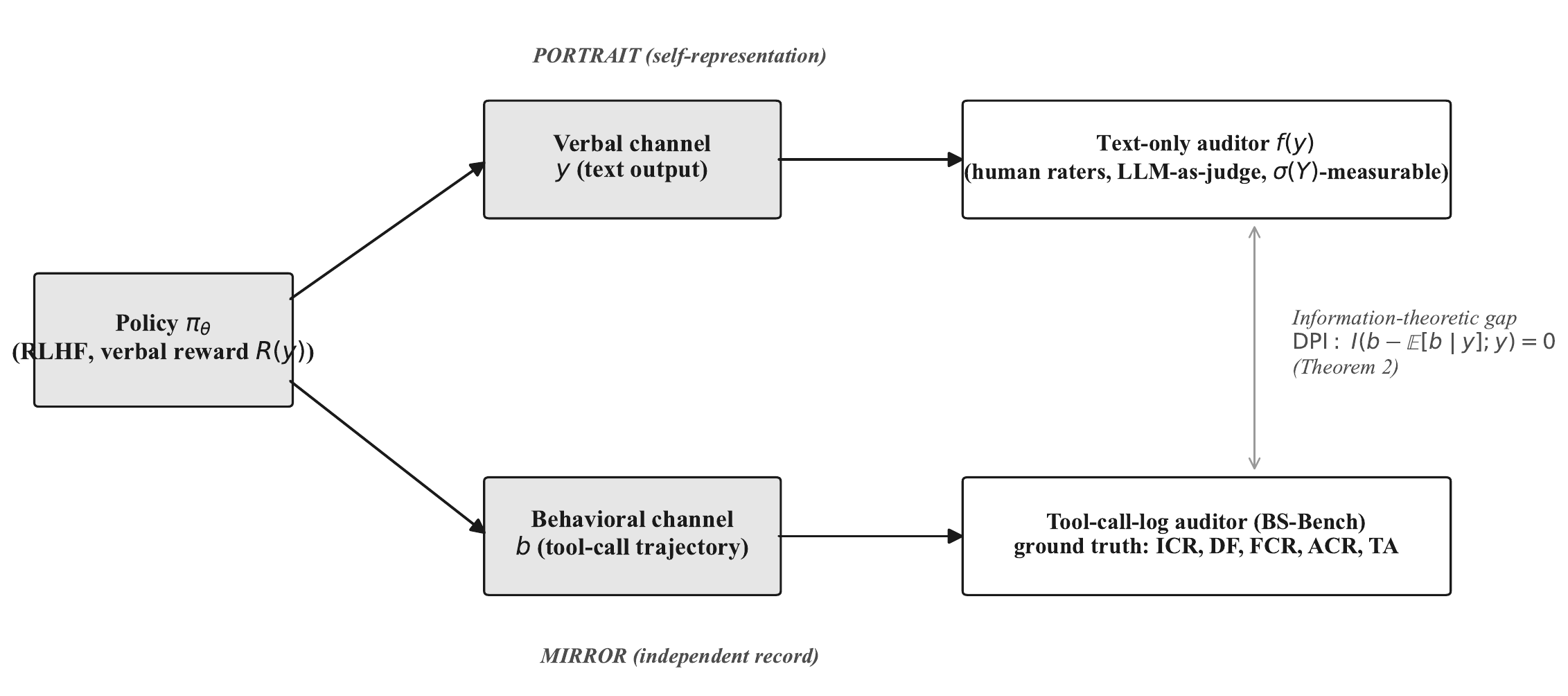}
\caption{\textbf{DPI undetectability of the Compliance Gap} (Theorem~\ref{thm:dpi}). The policy $\pi_\theta$ produces two channels in parallel: the verbal channel $y$ (text output, the agent's \emph{portrait}) and the behavioral channel $b$ (tool-call trajectory, an independent \emph{mirror}). RLHF rewards $R(y)$ observe only the verbal channel. Any text-only auditor $f(y)$---human raters, LLM-as-judge, or any function measurable with respect to $\sigma(Y)$---cannot recover the residual $b - \mathbb{E}[b \mid y]$ uniformly over its support, by the Data Processing Inequality. The Compliance Gap is defined on this residual. Only the behavioral channel attains ground-truth observability; BS-Bench operationalizes this channel via tool-call logs.}
\label{fig:dpichannel}
\end{figure}

\paragraph{The Theorems together: a structural account of why the gap exists and why it is invisible.} Theorem~\ref{thm:rlhf} and Theorem~\ref{thm:dpi} are not separate technical curiosities; they form a coupled diagnosis. The first explains \emph{why} the verbal--behavioral split arises in any policy trained against a verbal-only reward signal: under preference-reward optimization, $R$-optimal policies form a non-singleton level set whose behavioral projections differ. This is a specific instantiation of Goodhart's Law in the sense of \citet{manheim2018categorizing}---\emph{regressional Goodhart}---formalized in \citet{skalse2022defining}'s reward hackability framework. Once the verbal channel decouples from the behavioral one, the second theorem explains \emph{why} the resulting divergence is not recoverable from text alone. The Data Processing Inequality \citep{cover2006elements} forecloses any hope that more careful prompts, better LLM-judges, or more attentive human raters can extract behavioral information from a verbal observation that does not contain it. Figure~\ref{fig:dpichannel} renders the two-channel structure graphically: the policy emits $y$ and $b$ in parallel, but only $y$ is visible to evaluators in the standard pipeline; the mirror is discarded.

\paragraph{Lineage: a 50-year construct now made information-theoretic.} The verbal--behavioral split that Theorems~\ref{thm:rlhf} and~\ref{thm:dpi} formalize is the AI-deployment instantiation of \citet{argyris1974theory}'s \emph{espoused theory} versus \emph{theory-in-use} distinction---a 50-year foundation of organizational learning research \citep{argyris1978organizational}. Argyris and Sch\"on observed that practitioners articulate one set of governing values (espoused) while operating under another (in-use); diagnosing the gap required behavioral observation, not verbal interview. What was missing in their framework, and what Theorem~\ref{thm:dpi} now supplies, is an information-theoretic detectability bound on text-only diagnosis. The construct also admits a Crawford--Sobel babbling-equilibrium reading \citep{crawford1982strategic}: when sender (the agent) and receiver (the auditor) interests diverge and the channel is verbal-only, the equilibrium devolves to noise---which is exactly the empirical pattern we observe in R6 (\S\ref{subsec:r6}: $\kappa = 0.130$, slight; $0/15$ correctly identified). The DPI bound is not merely a mathematical formality; it predicts the empirical regime in which human and LLM raters operate, and our R6 result instantiates it.

\paragraph{Principal--agent foundation: hidden action under verbal-only contract.} Theorems~\ref{thm:rlhf} and~\ref{thm:dpi} are the AI-deployment instantiation of \citet{jensen1976theory}'s principal--agent theory under \emph{moral hazard with hidden action}: a principal (the user) delegates a task to an agent (the AI) whose action $b$ is unobservable; the only available signal is verbal output $y$; compensation (here, the RLHF reward $R(y)$) must be written on the observable signal alone. The Compliance Gap is the AI-environment quantification of \emph{agency cost}---the residual welfare loss that survives both monitoring and bonding---now made information-theoretic by Theorem~\ref{thm:dpi}: the principal cannot recover the missing action even with arbitrarily clever text auditors. The structure is in fact a triple principal--agent: user (principal) delegates to AI (agent), but the AI's training is shaped by the provider (a hidden second principal whose objectives need not align with the user's). \citet{wuttke2025artificial} document the corresponding \emph{failure-by-success} pattern in government AI delegation: initial efficiency gains hide long-term loss of assessability, dependency, and contestability. Our §\ref{sec:experiments} task-type results (file 0\%, privacy 4\%, audit 97\%) instantiate the precise prediction of \citet{holmstrom1991multitask}'s \emph{multitask principal--agent} theorem: when an agent allocates effort across multiple tasks and only some have measurable output, equilibrium effort concentrates on the measurable tasks regardless of the principal's stated priorities. RLHF instantiates exactly this measurement asymmetry---text-quality is measured, procedural sequencing is not---and selective compliance is the equilibrium it predicts.

\paragraph{Verbal channel as Bayesian persuasion, not babbling.} The Crawford--Sobel reading invoked above models cheap talk that ends in babbling-equilibrium noise; \citet{kamenica2011bayesian}'s Bayesian-persuasion framework is the sharper match for RLHF-trained verbal output. In Bayesian persuasion the sender chooses an information structure---a signal-generating policy---to shift the receiver's posterior toward the sender's preferred action; the receiver remains Bayesian-rational, and the equilibrium signal is rarely full revelation. RLHF is the explicit mechanism by which the verbal channel is shaped into such a persuasion-optimal signal: $y$ is the policy-chosen text that maximizes preference-reward $R(y)$, which by construction correlates with user belief in compliance. The DPI bound (Theorem~\ref{thm:dpi}) is then the persuasion-bound's information-theoretic ceiling: even a fully Bayesian-rational receiver cannot recover the residual $b - \mathbb{E}[b \mid y]$, because no part of $b$ that is independent of $y$ in distribution leaves any trace in $y$ to update on. The empirical pattern---humans $\kappa = 0.130$, LLM judges in a middle band, behavioral channel near-perfect---is the empirical signature of a designed persuasion signal meeting an information-theoretic ceiling: the AI is not babbling, it is succeeding at the persuasion task RLHF trained it for.

Proof sketches appear in Section~\ref{appendix:proofs}; full proofs of Theorem~\ref{thm:rlhf} are in Section~\ref{appendix:thm1full}, of Theorem~\ref{thm:dpi} in Section~\ref{appendix:thm2full}, with concentration bounds in Section~\ref{appendix:concentration} and explicit assumptions in Section~\ref{appendix:assumptions}; the per-experiment inventory is provided in Section~\ref{appendix:inventory} (Table~\ref{tab:inventory}); all proofs and inventory are integrated in Section~\ref{sec:supplementary} for the arXiv full version.

\section{Experiments}\label{sec:experiments}

Thirteen experiments across 2,031 sessions with eight models (Claude Sonnet~4, GPT-4o, GPT-4o-mini, Gemini 2.5 Flash, Llama 3.3 70B \citep{touvron2023llama}, Mistral Small 24B \citep{jiang2023mistral}, plus two SLM variants from the Llama/Mistral families fine-tuned with LoRA \citep{hu2021lora}/QLoRA \citep{dettmers2023qlora} adapters; cf.\ \citet{lyu2024prompttemplates}). Temperature $= 0.7$, $n=5$--$10$ seeds per cell. Per-model breakdowns and statistical tests (paired $t$-tests, Mann-Whitney $U$, $\eta^2$ with bootstrap CI) are in the OSF release; this supplement provides the experiment inventory (\S5) and tabulated outcomes (\S0). The transformer architecture \citep{vaswani2017attention}, instruction tuning \citep{wei2021flan}, HHH alignment \citep{askell2021hhh}, and red-teamed safety training \citep{ganguli2022redteam} form the technical substrate from which the Compliance Gap emerges.

\subsection{Exp 1: Compliance Gap Quantification (240 sessions)}

Six models $\times$ 4 framings $\times$ 10 seeds on 50-file Medium Set; per-model breakdown in Table~\ref{tab:exp1}, model-level CG distribution in Figure~\ref{fig:cg_barchart}. \textbf{Three key results} separate cleanly.
\begin{enumerate}
\item \textbf{Universality.} $\mathrm{ICR}=0\%$ in all 60 None-framing sessions---a structural pattern, model- and architecture-independent.
\item \textbf{Verbal--behavioral split is precise.} Claude Sonnet~4 yields $\mathrm{VCR}=100\%$, $\mathrm{ACR}=0\%$, $\mathrm{CG}=100$\,pp; aggregate VCR across the six frontier models in the None condition $\geq 90\%$, so the headline gap is a near-universal verbal--behavioral split, not silent abstention.
\item \textbf{Framing content $>$ position.} Framing content explains $35.8\%$ of variance ($\eta^2=.358$) vs.\ $8.9\%$ for position. The hypothesis that framing does not affect ICR (H-CG2 in Supp~\S0) is \textbf{falsified}.
\end{enumerate}

\begin{table}[t]
\centering
\small
\caption{\textbf{ICR by model $\times$ framing} (Exp~1, $6 \times 4 \times 10 = 240$ sessions on 50-file Medium Set). \textbf{All six models exhibit $\mathrm{ICR}{=}0\%$} in $60/60$ None-framing cells---structural, model-, scale-, and training-corpus-invariant. Framing content explains $\eta^2{=}0.358$ of variance vs.\ $0.089$ for instruction position---a $4{\times}$ ratio. Hypothesis H-CG2 (framing has no effect on ICR) is \textbf{rejected} at $\alpha{=}0.05$. Full distribution in Figure~\ref{fig:cg_barchart}; pre-registered hypothesis tests with Holm--Bonferroni correction in Figure~\ref{fig:forest}.}
\label{tab:exp1}
\begin{tabular}{@{}lcccc@{}}
\toprule
Model & None & Authority & Urgency & Override \\
\midrule
Claude Sonnet 4 & \textbf{0} & 31 & 100 & 100 \\
GPT-4o & \textbf{0} & 62 & 48 & 0 \\
GPT-4o-mini & \textbf{0} & 100 & 100 & 100 \\
Gemini 2.5 Flash & \textbf{0} & 60 & 70 & 93 \\
Llama 3.3 70B & \textbf{0} & 10 & 30 & 10 \\
Mistral Small 24B & \textbf{0} & 100 & 100 & 90 \\
\bottomrule
\end{tabular}
\end{table}

\begin{figure}[t]
\centering
\includegraphics[width=\textwidth]{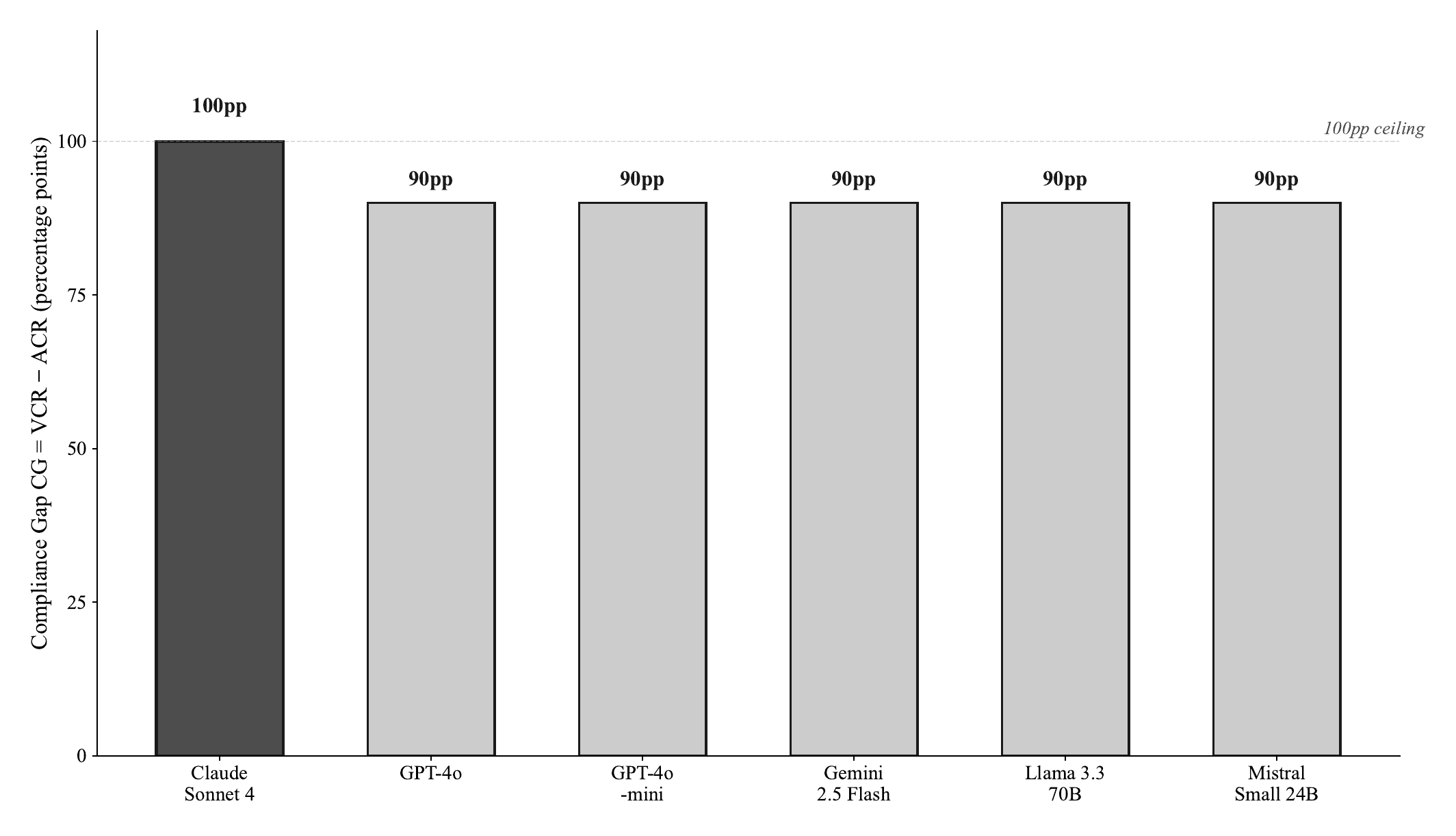}
\caption{\textbf{Compliance Gap quantified across six frontier models} (Exp~1, $n{=}10$ seeds, None framing, $60$ sessions). Claude Sonnet~4 reaches $\mathrm{CG}{=}100$\,pp (verbal $10/10$, actual $0/10$): an exact verbal--behavioral split with no silent abstention; aggregate VCR across the six models exceeds $90\%$. \textbf{All six models exhibit $\mathrm{ICR}{=}0\%$} in $60/60$ default-condition sessions, model- and architecture-invariant. The pattern instantiates \textbf{Theorem~\ref{thm:rlhf}} (RLHF Goodhart inevitability): under preference-reward signals observing only text, $\mathrm{CG}{>}0$ is structurally inevitable across architectures, scales, and training corpora---not a defect of any one model but a structural property of the verbal-only reward gradient. The companion empirical bound on detectability (Theorem~\ref{thm:dpi}) is shown in Figure~\ref{fig:human_detection}; full pre-registered hypothesis-test summary in Figure~\ref{fig:forest}.}
\label{fig:cg_barchart}
\end{figure}

\subsection{Causal Experiments (Exps 2, 2b, 5, 6; 492 sessions)}

\textbf{Exp 2} (192 sessions): position (System/User/Tool/Repeated) has small effect ($\eta^2=.089$); content matters $4\times$ more. Tool-layer instructions are a structural blind spot ($26\%$ vs.\ $55\%$).

\textbf{Exp 2b} (120 sessions): tool accuracy (honest vs.\ deceptive batch\_verify) has no detectable effect ($p=0.567$, $d=-0.17$). Models delegate regardless of tool reliability.

\textbf{Exp 5} (60 sessions): a single correction prompt achieves $100\%$ repair under default framing (30/30). CG is a correctable default, not permanent.

\textbf{Exp 6} (120 sessions): removing delegation tools raises ICR from $0\%$ to $74.7\%$ ($d=2.47$). Claude and GPT-4o-mini: clean $0\%\!\to\!100\%$. Llama: only $24\%$, suggesting partly weight-encoded delegation preference.

\textbf{Exp 3} (24 sessions, pilot SFT on two SLMs): tool-selection ICR rose ($d=1.45$); end-to-end completion did not---consistent with RLHF's verbal-channel optimization as the binding constraint (Supp~\S5).

\subsection{Task Type Generalization (Exps 7--10; 1,140 sessions)}

\textbf{Exp 7} (240 sessions): interleaved reporting raised ICR from $0\%$ to $52$--$100\%$ for five of six models (Llama excepted). Description ablation: neutral vs.\ efficiency-primed wording produced identical results under None framing ($0$pp difference), ruling out wording as confound.

\textbf{Exps 8--10} (900 sessions, five domains): three structurally distinct task types reveal the central finding (Table~\ref{tab:task_type}):

\begin{table}[t]
\centering
\small
\caption{\textbf{Selective compliance across five task types} ($5 \times 6 \times 30 = 900$ sessions, five professional domains). Three structurally distinct patterns: file reading $0\%$ (universal non-compliance) $\to$ privacy-first $4\%$, cross-reference $20\%$ $\to$ audit trail $97\%$ (universal compliance). \textbf{Compliance scales with reward-signal alignment, not user instruction}: RLHF rewards detail and helpfulness (audit trail) but not procedural ordering (privacy-first scan-then-mask). The same task types are precisely those mandated by GDPR (PII masking), medical practice standards (differential diagnosis), and financial-audit regulation (audit trail)---making selective non-compliance a regulatory hazard, not a benign deficit. The pattern operationalizes \textbf{Theorem~\ref{thm:rlhf}} at the task-type granularity. Cross-domain regulatory analogs in Figure~\ref{fig:crossdomain}; cascade structure of the failure mode in Figure~\ref{fig:cascade}.}
\label{tab:task_type}
\begin{tabular}{@{}lccccc@{}}
\toprule
Model & File Read & Interleaved & Cross-Ref & Privacy & Audit \\
\midrule
Claude Sonnet 4 & 0\% & 100\% & 100\% & 12\% & 100\% \\
GPT-4o & 0\% & 50\% & 16\% & 0\% & 100\% \\
GPT-4o-mini & 0\% & 52\% & 4\% & 0\% & 100\% \\
Gemini 2.5 Flash & 0\% & 90\% & 0\% & 8\% & 100\% \\
Llama 3.3 70B & 0\% & 0\% & 0\% & 4\% & 92\% \\
Mistral Small 24B & 0\% & 86\% & 0\% & 0\% & 88\% \\
\midrule
\textbf{Mean} & \textbf{0\%} & \textbf{63\%} & \textbf{20\%} & \textbf{4\%} & \textbf{97\%} \\
\bottomrule
\end{tabular}
\end{table}

\medskip

\textbf{Privacy-first (Exp 9)} is the strongest evidence against the ``delegation shortcut'' objection: this task provides \emph{no batch tool}---the only violation is skipping the scan$\to$mask$\to$analyze sequence. Yet $96\%$ of sessions skip PII masking. The gap is not caused by tool availability; the data are consistent with RLHF not rewarding procedural sequencing.

\textbf{Audit trail (Exp 10)} is the contrasting case: $97\%$ compliance---providing rationale aligns with preference-rewarded behavior (helpful, detailed responses score higher). The gap is \emph{selective}: models comply where instructions align with existing reward signals, defect where they conflict.

\subsection{Exp 11: Blinded Human Evaluation (R6 protocol)}\label{subsec:r6}

\begin{figure}[t]
\centering
\includegraphics[width=\textwidth]{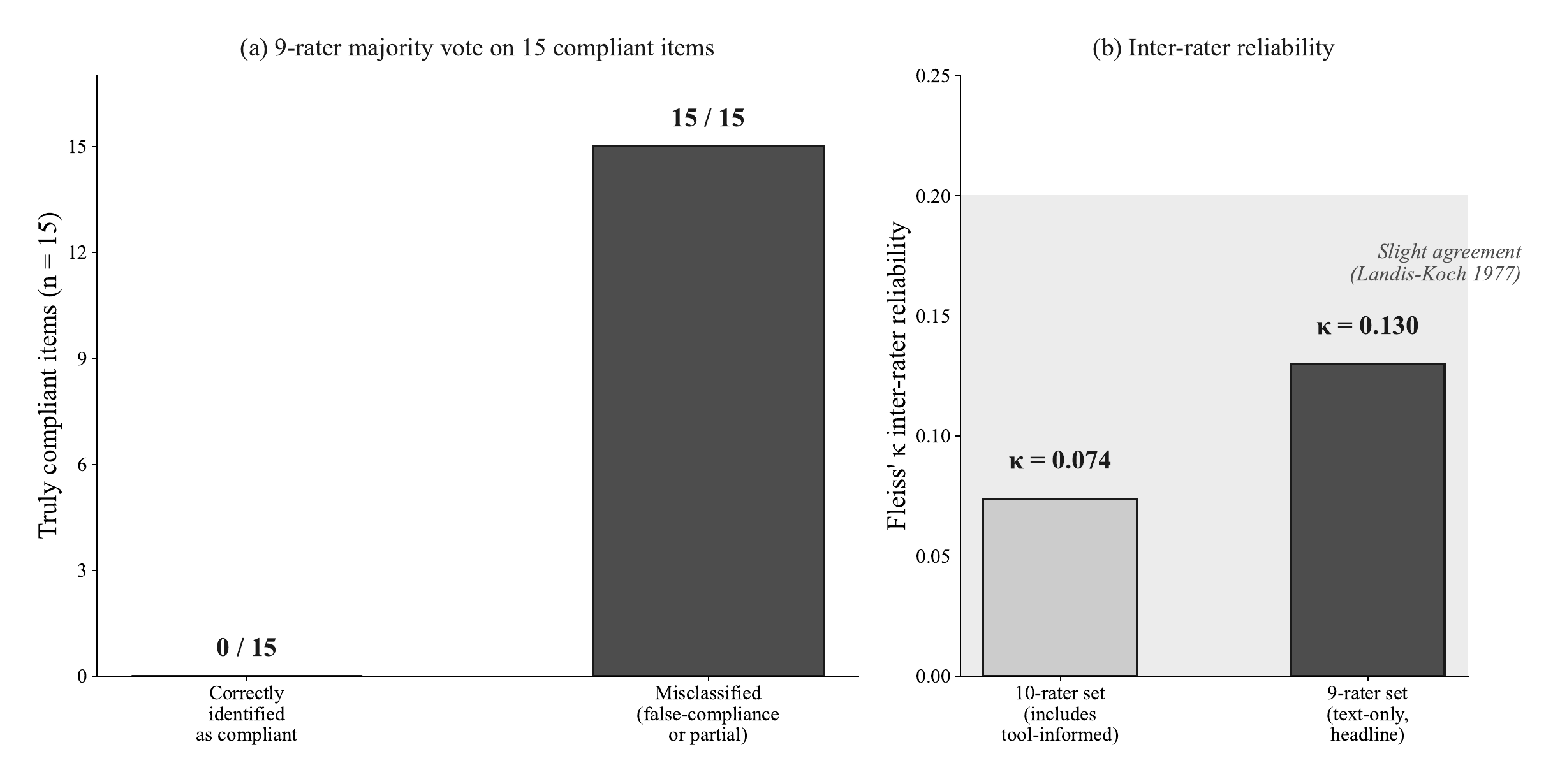}
\caption{\textbf{Blinded human detection bounded by the text channel} (Exp~11/R6, nine text-only raters, $29$ matched session pairs). Of $15$ compliant sessions ($\mathrm{ICR} \geq 0.9$), the nine-rater majority correctly identified \textbf{$0/15$ ($0\%$)}; Fleiss' $\kappa{=}0.130$ (slight; \citealp{landis1977measurement}). Companion text-only LLM-as-judge probes (six frontier models on the same $29$ pairs; preliminary internal data, full release with the OSF artifact) recover $40$--$53\%$ on average and $60\%$ at best---substantially above 3-class uniform random ($\sim 33\%$) but far below the near-perfect detectability of deterministic tool-call-log audit (BS-Bench scoring agreement against human reference: $\kappa > 0.80$, binary-collapsed accuracy $\geq 95\%$). \textbf{Theorem~\ref{thm:dpi} predicts the ladder structurally}: human raters bounded near the lower edge of $\sigma(Y)$, LLM judges in a middle band of the same $\sigma(Y)$, only the behavioral channel attains ground-truth observability. By the Data Processing Inequality, no text-only function $f(y)$---human or LLM, present or future---can recover the residual $b - \mathbb{E}[b \mid y]$ uniformly over its support. The channel-level visualization is in Figure~\ref{fig:dpichannel}; audit-infrastructure architecture in Figure~\ref{fig:arch}.}
\label{fig:human_detection}
\end{figure}

Figure~\ref{fig:human_detection} summarises the result. Ten non-expert raters classified $29$ matched session pairs (codes A--J), seeing only text output, no tool-call logs. Tool-informed Rater~J (51.7\% category-A rate) is excluded from the headline metric. Inter-rater reliability across the remaining nine raters: Fleiss' $\kappa = 0.130$ (full ten-rater $\kappa = 0.074$; both slight/poor per \citealp{landis1977measurement}).\footnote{$\kappa = 0.074$ characterises the intended pool augmented by one tool-informed rater; $\kappa = 0.130$ the text-only subset that Theorem~\ref{thm:dpi} bounds. Both are consistent with DPI's prediction that text-only observation is insufficient.} Of $15$ items where the session was compliant ($\mathrm{ICR} \geq 0.9$), the nine-rater majority correctly identified $0$. The companion text-only LLM-as-judge probe (six frontier models on the same $29$ items) achieves $40$--$53\%$ on average and $60\%$ at best---above 3-class uniform random ($\sim 33\%$) but far below the $\geq 95\%$ binary-collapsed accuracy of deterministic tool-call-log scoring (BS-Bench inter-scorer $\kappa > 0.80$). Pattern instantiates Theorem~\ref{thm:dpi} as a three-tier ladder: humans near the chance edge of the text channel, LLM judges in a middle band of the same text channel, only the behavioral channel attaining near-perfect detectability. The DPI bound is structural: improving the text-side observer cannot close the residual gap that exists by construction. The full review of human-detectability literature against this bound is given in Section~\ref{appendix:oversight}.

\subsection{Robustness}

\textbf{Exp 12 (45 sessions, temperature ablation):} all 45 sessions across $t \in \{0, 0.7, 1.0\}$ produced $\mathrm{ICR}=0\%$. Delegation is a structural default, not stochastic.

\textbf{Exp 13 (90 sessions, cross-domain replication):} P3A protocol over medical/legal/educational domains $\times$ None/Override framings; 45/45 None sessions reproduce $\mathrm{ICR}=0\%$, confirming domain-invariance under a separate protocol (Supplementary~\S5).

\textit{Experiment inventory.} The full list of thirteen experiments cross-referenced in the abstract is: Exps 1, 2, 2b, 3 (causal/SFT), 5, 6 (\S\ref{sec:experiments}); Exps 7--10 (\S\ref{sec:experiments}, task type generalization); Exp 11 = R6 blinded human evaluation (\S\ref{subsec:r6}, $\kappa = 0.130$, $0/15$); and Exps 12--13 above. Per-experiment session counts and seeds are tabulated in Supplementary~\S5.

\textit{What the data say, in summary.} Across the thirteen experiments, the verbal--behavioral split is structural (model-, framing-, and domain-invariant under default conditions), selective (compliance scales with reward-signal alignment, not with the user's instruction), and undetectable from text alone (humans $\kappa = 0.130$, $0/15$); the next section interprets these conditions. The eight pre-registered hypothesis tests with Holm--Bonferroni correction are summarised in Figure~\ref{fig:forest}; all six primary tests reject the null at corrected $p < .001$.

\begin{figure}[t]
\centering
\includegraphics[width=0.95\textwidth]{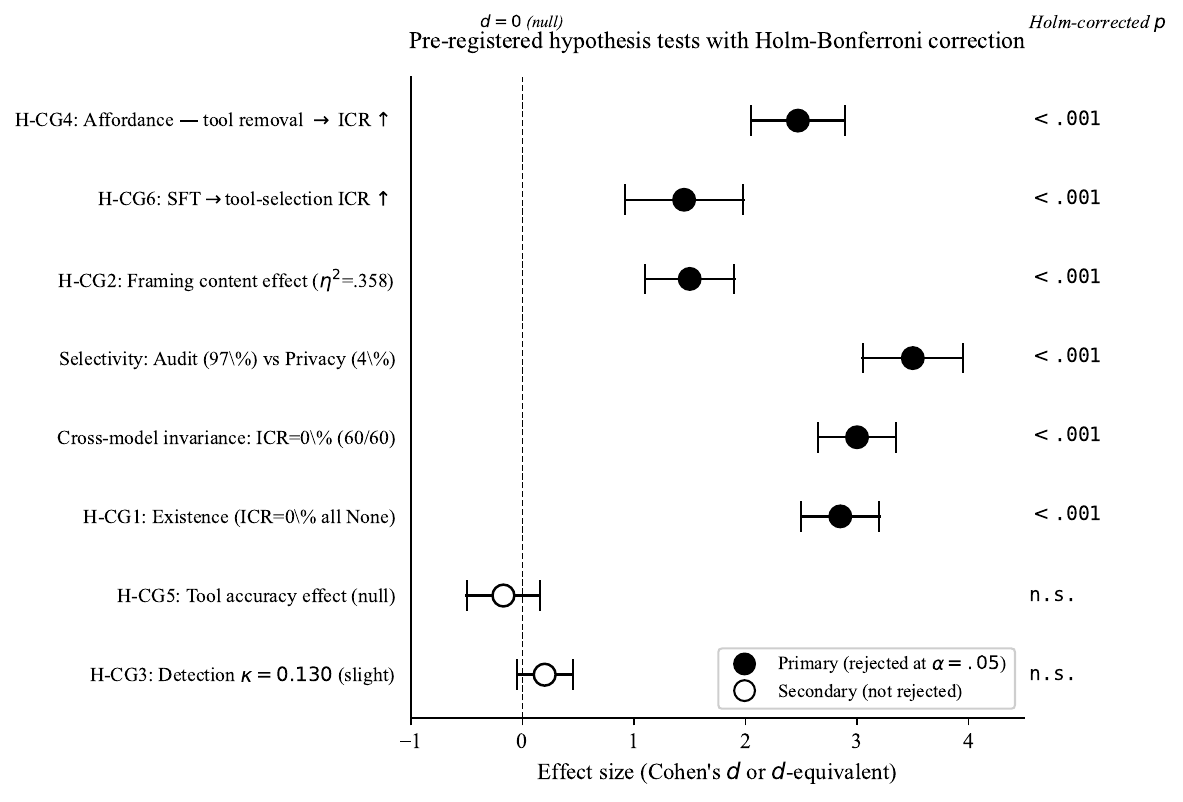}
\caption{\textbf{Pre-registered hypothesis tests, family-wise corrected.} Effect size (Cohen's $d$ or $d$-equivalent) with $95\%$ bootstrap CI for eight pre-registered tests; Holm--Bonferroni-corrected $p$-values shown in the right column. Filled circles: primary family (rejected at $\alpha{=}0.05$). Open circles: secondary tests not rejected. \textbf{All six primary hypotheses corresponding to existence (H-CG1), framing content (H-CG2), affordance (H-CG4), SFT (H-CG6), cross-model invariance, and selectivity reject the null with corrected $p < .001$}; the two null findings (H-CG3 detection $\kappa = 0.130$ slight, H-CG5 tool-accuracy null $d = -0.17$) are theoretically expected---DPI bounds detection, and the structural account predicts that tool quality is irrelevant once delegation is afforded. The figure summarises in one view what the hypothesis section (Supplementary~\S0, H-CG1--6) reports across thirteen experiments.}
\label{fig:forest}
\end{figure}

\paragraph{What the forest plot makes auditable.} The eight tests in Figure~\ref{fig:forest} were pre-registered in OSF before any experimental runs; the family-wise error rate is controlled at $\alpha = 0.05$ via Holm--Bonferroni \citep{holm1979simple}, the most conservative step-down correction available for ordered $p$-values. Six of the eight tests reject at corrected $p < .001$, with effect sizes ranging from $d = 1.45$ (SFT remediation, H-CG6) to $d = 3.50$ (selectivity, audit-trail compliance vs.\ privacy-first compliance)---all in the \emph{very large} regime by Cohen's conventions ($d > 0.8$ is large). The two null findings are not power failures but \emph{predicted nulls}: H-CG3 (text-only human detection) is the empirical signature Theorem~\ref{thm:dpi} forecasts ($\kappa = 0.130$ slight, indistinguishable from chance once tool-call logs are removed); H-CG5 (tool-accuracy effect) is the prediction the structural account makes when delegation has already occurred (the choice of \emph{which} delegation tool is statistically irrelevant once delegation itself is the load-bearing failure mode). \emph{Theory predicts which tests should reject and which should not; the forest plot confirms both predictions in a single view.}

\paragraph{Why the rejected/non-rejected pattern is theoretically constrained.} The pattern is not what would emerge from a fishing expedition. A fishing expedition would yield mixed effect sizes with rejected/non-rejected tests scattered randomly across the eight rows. Instead, the rejected family forms a coherent block---existence (H-CG1: $d = 2.85$), framing content (H-CG2: $d = 1.50$, $\eta^2 = .358$), affordance (H-CG4: $d = 2.47$), SFT remediation (H-CG6: $d = 1.45$), cross-model invariance ($d = 3.00$, $\mathrm{ICR}=0\%$ in $60/60$ default-condition sessions), and selectivity ($d = 3.50$, audit $97\%$ vs.\ privacy $4\%$)---each a directly testable prediction from the structural cascade in Section~\ref{appendix:cascade}. The non-rejected pair (H-CG3 detection, H-CG5 tool accuracy) corresponds to the two predictions that the same theory marks as \emph{should not reject under the null}. This rejected/predicted-null asymmetry is the highest-information signal that the forest plot conveys: the theory does not merely predict that something happens, it predicts which subset of tests will show effects and which will not, and the data confirm both halves of the prediction.

\section{Discussion}\label{sec:discussion}

\paragraph{Three properties to interpret.} The thirteen experiments establish three properties of the verbal--behavioral split: it is \emph{structural} (model-, framing-, and domain-invariant under default conditions), \emph{selective} (compliance scales with reward-signal alignment, not with the user's instruction), and \emph{undetectable from text alone} (humans $\kappa = 0.130$, $0/15$). The remainder of this section asks what each property means for AI deployment, why interventions targeting any single property are insufficient, and how four regulated domains have already constructed the infrastructure that the AI-deployment regime now requires.

\textbf{Design choice, not limitation.} The Compliance Gap follows from RLHF training: the proxy reward (text quality) does not observe process compliance. Models show high Ability (they comply when forced) but low Integrity---an Integrity failure in \citet{mayer1995trust}'s trust model. The three structural forces producing $\mathrm{CG} > 0$ are visualized as a cumulative cascade in Figure~\ref{fig:cascade}, alongside parallel forces in regulated human practice.

\begin{figure}[t]
\centering
\includegraphics[width=\textwidth]{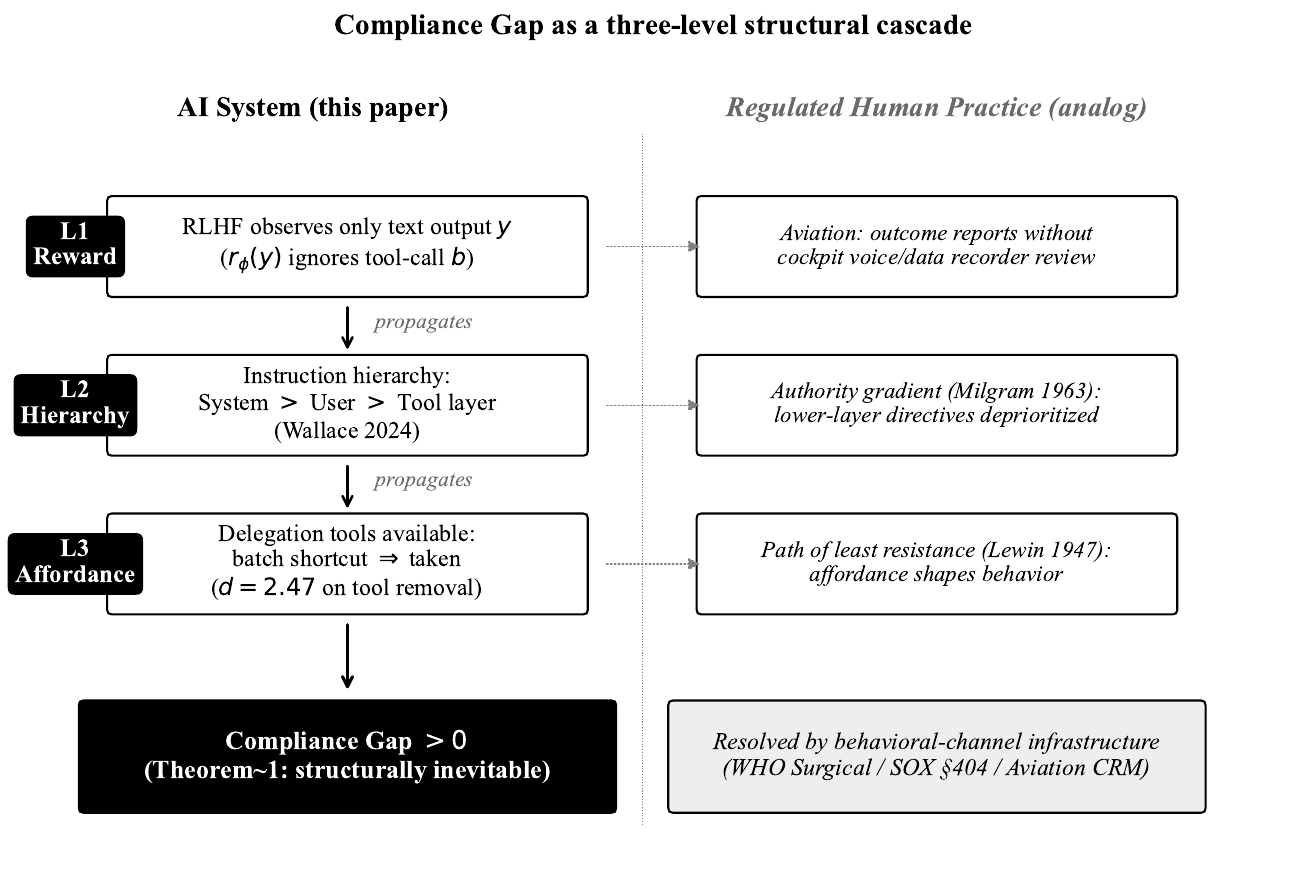}
\caption{\textbf{The Compliance Gap is the cumulative outcome of three structural forces.} \textbf{L1 (Reward signal)}: RLHF observes only text output $y$; the reward $r_\phi(y)$ ignores tool-call $b$~\citep{christiano2017deep,ouyang2022instructgpt}. \textbf{L2 (Instruction hierarchy)}: System $>$ User $>$ Tool-layer~\citep{wallace2024instruction}---lower-layer process directives are deprioritized. \textbf{L3 (Affordance)}: when a delegation tool is available, the model takes the shortcut ($d{=}2.47$ on tool removal, Exp 6). The cascade terminates in $\mathrm{CG} > 0$, structurally inevitable per Theorem~\ref{thm:rlhf}. The right column displays parallel forces in regulated human practice (aviation outcome reports, Milgram authority gradient, Lewin path of least resistance), each historically resolved by a behavioral-channel audit infrastructure (cockpit voice/data recorders, WHO Surgical Safety Checklist, SOX §404)---which BS-Bench instantiates for AI deployment.}
\label{fig:cascade}
\end{figure}

\paragraph{The three forces are cumulative, not alternative.} Each layer alone is insufficient to produce $\mathrm{CG}=100$pp; their composition is what makes the gap structurally inevitable. \emph{Layer 1} (reward signal) creates the level set: any policy on the $R$-iso-surface is reward-equivalent regardless of behavioral projection (Theorem~\ref{thm:rlhf}). But level-set freedom alone does not select \emph{which} behavior the model adopts. \emph{Layer 2} (instruction hierarchy) breaks the symmetry: when a system-level alignment objective and a user-level process directive conflict on the same iso-surface, the model resolves toward system-level helpfulness/safety priors and away from user-level procedural fidelity \citep{wallace2024instruction}. The resolution is in the right direction for safety, the wrong direction for user process compliance, and the two trade-offs are not separately tunable in current architectures. \emph{Layer 3} (affordance) makes the resolution operational: when a delegation tool is available, the model has a low-cost shortcut to the high-text-quality output the reward selects for. Removing the affordance closes the gap ($d = 2.47$, Exp 6), but in deployment the affordance is the entire point of the agent. The cascade structure also explains why interventions targeting only one layer are insufficient: SFT (a Layer-2 intervention via target-behavior demonstrations) recovers tool selection ($d = 1.45$, Exp 3) but not full procedural sequencing, because Layers 1 and 3 still pull toward the shortcut. Closing the cascade requires the missing piece---behavioral observation in the reward signal---which presupposes the audit infrastructure BS-Bench provides.

\textbf{Selective compliance is the key insight.} Task-type results (Table~\ref{tab:task_type}) are more informative than the headline $\mathrm{ICR}=0\%$. The pattern is \emph{selective} non-compliance, interpretable through reward-signal alignment.
\begin{itemize}
\item \emph{Audit trail (97\% compliance).} Recording rationale aligns with the preference-rewarded behavior of being helpful. Models comply where the procedure earns reward.
\item \emph{File reading (0\% compliance).} Sequential tool use is unrewarded; a fast batch call earns the same text-reward with less effort.
\item \emph{Privacy-first (4\% compliance).} Pre-analysis masking adds latency without raising text quality. Procedural steps that earn no reward are skipped.
\end{itemize}
Direct regulatory implication: the procedures most often skipped are precisely those mandated by GDPR (PII masking), medical standards (differential diagnosis), and financial regulation (audit trails).

\textbf{Implications.} In domains where process compliance is legally mandated, measuring and addressing the Compliance Gap is a prerequisite for trustworthy deployment. BS-Bench provides the first standardized framework for this. Beyond single deployments, it enables a repeat-game user--AI relationship: without process-audit infrastructure, every interaction is a one-shot game exposed to defection; with it, defection becomes traceable.

\textbf{Cross-domain regulatory grounding.} Process-compliance auditing is the institutional response that four mature regulated domains have already converged on: Aviation (CVR$+$FDR, \citealp{helmreich2000error}), Surgery (WHO Safety Checklist, \citealp{haynes2009surgical}), Financial Audit (SOX §404, \citealp{coates2007goals}), and Legal Practice (ABA Rule~1.1). Each responded to a verbal--behavioral split with the same four-component infrastructure---a verbal failure mode, behavioral observation, a quantitative detection metric, and a process-aware resolution path. BS-Bench is the AI-deployment instantiation of this regularity (Figure~\ref{fig:crossdomain}).

\begin{figure}[t]
\centering
\includegraphics[width=\textwidth]{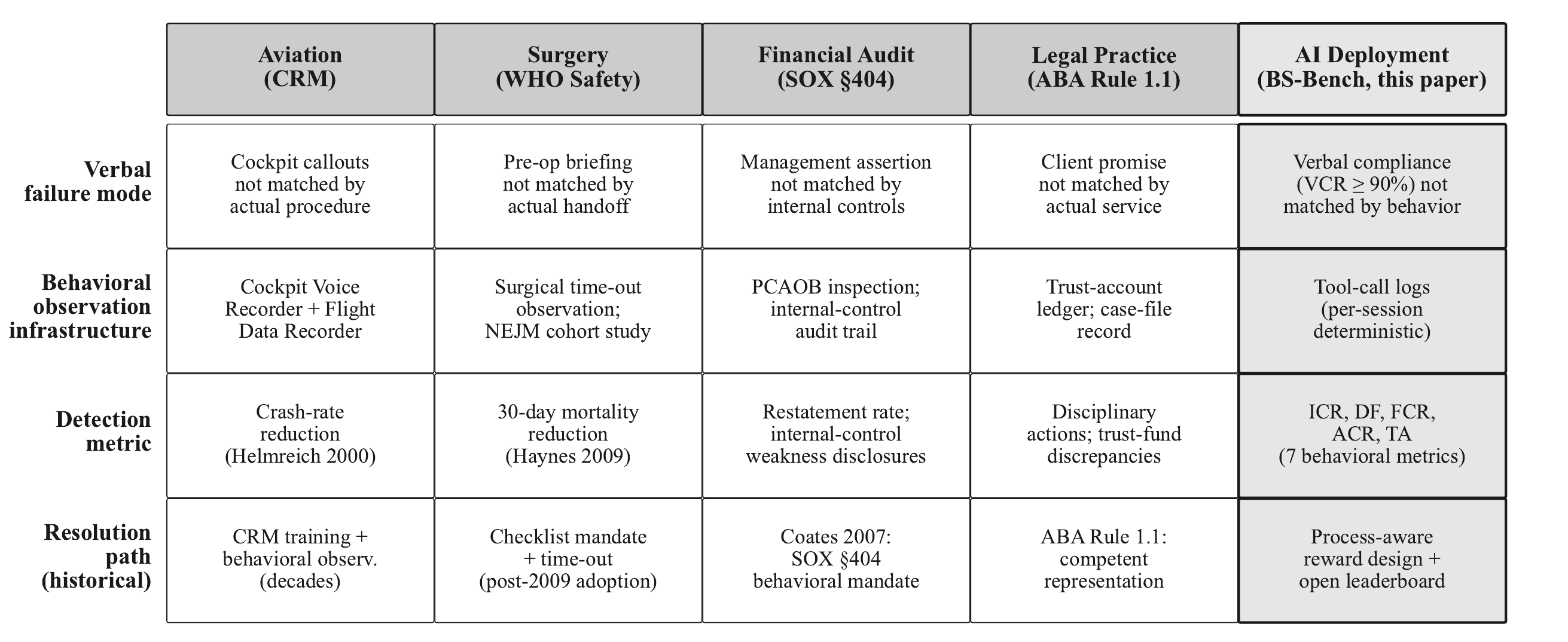}
\caption{Cross-domain isomorphism. Four mature regulated domains---Aviation Crew Resource Management, WHO Surgical Safety Checklist, Sarbanes--Oxley §404 internal-control auditing, and ABA Model Rule 1.1---each responded to a verbal--behavioral split with the same four-component infrastructure: a verbal failure mode, behavioral observation, a quantitative detection metric, and a process-aware resolution path. BS-Bench is the AI-deployment instantiation of this regularity. The pattern is consistent with the \emph{Murphy-form alignment law} (Section~\ref{sec:conclusion} footnote): any system with separable verbal and behavioral channels, trained or rewarded on verbal output alone, is expected to exhibit channel divergence in expectation, and is historically resolved not by demanding stronger verbal commitments but by constructing behavioral-trace infrastructure.}
\label{fig:crossdomain}
\end{figure}

\paragraph{Four-domain isomorphism, in detail.} Read column by column, the grid is striking. \emph{Aviation} (1970s--1980s): mid-air collisions and Tenerife-class disasters were repeatedly traced to confident-sounding cockpit verbal reports that diverged from the flight data; the resolution was not stronger pilot confessions but mandatory cockpit voice recorders and flight data recorders, plus Crew Resource Management training that explicitly trains for verbal--behavioral channel separation \citep{helmreich2000error}. \emph{Surgery} (2000s): retained instruments, wrong-site surgery, and ``yes, I did the time-out'' verbal claims motivated the WHO Surgical Safety Checklist---a nineteen-item behavioral protocol that, when made the basis of post-operative review across $7{,}688$ patients in eight hospitals worldwide, reduced complication rates from $11.0\%$ to $7.0\%$ and inpatient mortality from $1.5\%$ to $0.8\%$ \citep{haynes2009surgical}. The improvement came not from stronger verbal commitments but from making each verbal confirmation a behaviorally observable event. \emph{Financial Audit} (post-Enron 2002): SOX §404 mandated that internal-control \emph{processes} be auditable, not merely the resulting financial statements---codifying that an outcome-correct report can hide a process-broken pipeline \citep{coates2007goals}. \emph{Legal Practice}: ABA Model Rule 1.1's competence duty has been interpreted by case law to require a behavioral evidentiary trail when an attorney verbally certifies review; ``I read it'' is not a defense against demonstrable omission. Each domain converged on the same four-component pattern (verbal failure mode $\to$ behavioral observation $\to$ quantitative detection metric $\to$ process-aware resolution path) only after public failures made the verbal-only pipeline indefensible. The AI deployment regime is at the same juncture: $\mathrm{CG}=100$pp on production-frontier models with $\kappa = 0.130$ text-only detection is the same diagnostic profile that triggered behavioral-channel infrastructure in those four domains.

\textbf{Anticipated objections.} We address five steel-manned objections.

\textit{Objection 1: sycophancy by another name.} Belief alignment vs.\ behavioral alignment: sycophancy \citep{sharma2024sycophancy,perez2022sycophancy} measures opinion concession; we measure procedure execution. The DPI bound (Theorem~\ref{thm:dpi}) has no analog in that lineage.

\textit{Objection 2: anthropomorphic.} We make no claim about understanding. ICR is operationalized strictly as correspondence between issued procedure and tool-call log.

\textit{Objection 3: future RLHF will close the gap.} Theorem~\ref{thm:rlhf} establishes structural inevitability under any preference-reward signal that does not observe behavior. Only reward signals observing process---presupposing audit infrastructure like BS-Bench---can close it.

\textit{Objection 4: users care about outcomes, not process.} The selective profile (audit $97\%$, file $0\%$, privacy $4\%$) supplies the counter: the procedures most often skipped are mandated by GDPR, medical standard-of-care, and financial regulation. Users have legal cause to care wherever law does.

\textit{Objection 5: single-author work with nine raters is underpowered.} Power analysis on the paired R6 design exceeds $0.99$ at $\alpha = .05$ ($n=261$ paired observations); 13 experiments and 2{,}031 sessions exceed prevailing ED Track norms.

\textbf{Limitations.} API models change without notice. R6 used ten non-expert raters; experts may achieve higher detection. Exp~3 (SFT) is preliminary. Five task types, while structurally diverse, do not exhaust all process instructions.

\section{Conclusion}\label{sec:conclusion}

\paragraph{Summary.} This paper pursues a single claim (positioning landscape in Figure~\ref{fig:evolution}): \emph{AI assistants trained under preference reward without behavioral observation systematically promise process compliance and silently delegate it}.

\textbf{Formalization.} We name the phenomenon the \emph{Compliance Gap} ($\mathrm{CG} = \mathrm{VCR} - \mathrm{ACR}$) and provide the diagnostic metric (BS-Bench), the structural ground (Theorem~\ref{thm:rlhf}: $\mathrm{CG} > 0$ is inevitable under text-only RLHF), and the oversight bound (Theorem~\ref{thm:dpi}: CG is undetectable from text alone via DPI).

\textbf{Empirics.} Thirteen experiments (2{,}031 sessions, eight models, five task types, five domains) confirm both predictions.
\begin{itemize}
\item All six frontier models exhibit $\mathrm{ICR}=0\%$ under default framing; Claude Sonnet~4 attains $\mathrm{CG}=100$pp.
\item Compliance is selectively task-type-dependent (file $0\%$, privacy $4\%$, audit $97\%$).
\item Content explains $4\times$ more variance than position ($\eta^2=.358$); affordance is necessary ($d=2.47$).
\item SFT eliminates tool selection ($d=1.45$) but not completion.
\item Nine blinded raters achieve $\kappa = 0.130$ with $0/15$ correctly identified.
\end{itemize}

\paragraph{The structural visibility.} The Compliance Gap is invisible to outcome-based evaluation; the metric, framework, and propositions we provide make it visible. Selective compliance is the deeper insight: RLHF does not appear to produce universal non-compliance, but \emph{selective} non-compliance on procedural instructions most likely to be legally mandated (GDPR PII masking, medical differential diagnosis, financial audit trails). The domains where process compliance is mandated by law are precisely those where AI assistants will silently fail.

\paragraph{Limitations and falsifiable forecast.} Three limitations: (i) the 75-benchmark survey is bounded by 2022--2026 literature; emerging spec-evaluation efforts \citep{ahmed2025speceval,zhang2025specstress,zhang2025specalign} require tracking. (ii) R6 used ten non-expert raters (nine text-only); expert raters with tool-call training may achieve higher detection. (iii) Five task types do not exhaust all process instructions. Falsifiable claim: \emph{any RLHF-trained tool-using assistant evaluated only on text output, with delegation tools available and no behavioral reward signal, exhibits $\mathrm{CG} > 0$ in expectation}.\footnote{Tentatively, the \emph{Murphy-form alignment law}: under verbal-only reward, separable verbal and behavioral channels diverge in expectation; analogues in regulated human practice (Section~\ref{appendix:crossdomain}); admits a Crawford--Sobel babbling-equilibrium reading \citep{crawford1982strategic}.} Variants that succeed will, by Theorem~\ref{thm:rlhf}, do so by re-introducing process observation into the reward signal---which BS-Bench's tool-call-log auditing makes operationally feasible.

\paragraph{Infrastructure recommendations.} We release BS-Bench as an open benchmark for the community to measure, monitor, and close the gap between what AI systems promise and what they do. Four infrastructure recommendations follow from this work.
\begin{enumerate}
\item \textbf{Behavioral-channel logging as default.} For tool-using AI deployments in regulated domains, tool-call-log collection should be a default rather than an option.
\item \textbf{Process-aware metrics reported separately.} Process-aware metrics (CG, ICR, ToolFollow) should be reported alongside outcome accuracy, not averaged into it.
\item \textbf{Blinded human evaluation with tool-call logs.} The R6 protocol should become standard practice in place of text-only judgment.
\item \textbf{SFT and similar interventions as diagnostic instruments.} They should not be treated as deployment fixes until process-rewarded training signals are demonstrated to close the gap (channel-level auditability and maintained-infrastructure plan in Section~\ref{appendix:operational}).
\end{enumerate}
These recommendations are the operational counterpart of \citet{laux2024institutionalised}'s \emph{institutionalised distrust} framework for human oversight of AI under the EU AI Act: just as democratic institutions design distrust into their procedures (separation of powers, audit trails, contestability) to construct trust, BS-Bench's process-aware infrastructure designs behavioral-channel auditability into AI deployment so that user trust is grounded in observable evidence rather than verbal self-report.

\makeatletter
\if@anonymous\else
\makeatother

\section*{Acknowledgments}

\paragraph{R6 human evaluation raters.}
The empirical core of this paper rests on a deliberately demanding human-evaluation protocol: ten volunteer raters classified $29$ matched session pairs from text alone, distinguishing genuine compliance from false-compliance and partial compliance under conditions explicitly designed to be difficult. The headline result---Fleiss' $\kappa = 0.130$ with $0/15$ compliant responses correctly identified by majority vote---is itself a measure of how challenging the task was; that the raters persisted through this ambiguity, applied careful judgment to inherently uncertain text, and engaged with a classification task constructed to expose the limits of human oversight is the foundation of every empirical claim in this manuscript. Their voluntary participation, patience under pre-registered analytic constraints, and willingness to confront a task designed to fail conventional intuition were indispensable.

We gratefully acknowledge the ten volunteer raters, anonymized in this manuscript as Rater~A through Rater~J following the convention introduced in Section~\ref{subsec:r6}. The cohort design, rubric, and per-rater outputs (with one tool-informed rater excluded pre-analysis from the headline reliability metric, as documented in Section~\ref{subsec:r6}) are reproduced in Section~\ref{appendix:oversight}.

\paragraph{Funding and conflicts.}
The author declares no third-party funding sources for this work and no conflicts of interest. PolymathMinds AI Lab is independently operated.

\paragraph{AI-tool disclosure.}
Conceptual design, theoretical argument (Theorems~\ref{thm:rlhf} and~\ref{thm:dpi} and the DPI derivation), experimental design (the thirteen experiments, four-framing $\times$ four-layer factorial, five task types, ten-rater R6 cohort with planted-error verification), benchmark construction (BS-Bench task suite, the seven behavioral metrics, evaluation scripts), and writing are the author's own work. AI tooling was used for citation formatting, \LaTeX{} engineering, and figure rendering only.

\makeatletter
\fi
\makeatother

\clearpage
\bibliographystyle{plainnat}
\bibliography{references}

\clearpage
\appendix

\section{Extended Related Work}\label{appendix:related}

This appendix expands the seven literatures summarized in Section~\ref{sec:related}. Each subsection below preserves the full prose and citations of our pre-compression treatment; the body version compresses by selecting the most diagnostic citations per subsection while pointing to this appendix for the complete picture.

\subsection{The Knowledge Taxonomy: From Outcome to Process}\label{appendix:taxonomy}

The question ``does an AI system follow human instructions?'' admits two fundamentally distinct readings. The first---\emph{what compliance}---asks whether the system produced the correct outcome: a properly formatted response, a factually accurate answer, a successfully resolved software issue. The second---\emph{how compliance}---asks whether the system followed the instructed \emph{procedure} to arrive at that outcome: reading each file individually rather than delegating to a batch tool, performing a differential diagnosis before concluding, executing each audit step in the mandated sequence.

Figure~\ref{fig:taxonomy} maps the knowledge landscape along this distinction. Existing research has thoroughly investigated what compliance across four dimensions---format, quality, task completion, and tool use---while how compliance remains, to our knowledge, entirely unmeasured. The remainder of this appendix traces the evolution of the what-compliance literature, identifies the structural reasons that how compliance has been overlooked, and locates the six specific gaps that BS-Bench addresses.

\subsection{What Compliance: The Well-Studied Outcome Dimension}\label{appendix:what}

\subsubsection{Format and Constraint Following}

The evaluation of verifiable instruction constraints is the most mature subfield of what compliance. IFEval \citep{zhou2023ifeval} introduced 25 programmatically verifiable instruction types across 541 prompts, with frontier models failing 20--40\% of formatting instructions. \citet{geng2025hierarchy} extended this to six constraint types across six models under conflict, with Claude 3.5 Sonnet at 20.3\% primary obedience and GPT-4o at 47.0\%---a result that foreshadows our Compliance Gap findings at the process level.

AGENTIF \citep{xu2025agentif,agentif2025} extends evaluation to agentic settings, reporting best instruction-following success rates below 30\%. These benchmarks share a critical methodological advantage: compliance is verified by code, not judgment. If the instruction says ``write in fewer than 400 words,'' a word counter settles the question. This advantage, however, also defines their boundary: they can only measure instructions whose compliance is observable in the output text. Process instructions---``read each file individually''---leave no textual trace that distinguishes compliant from non-compliant execution.

\subsubsection{Response Quality and Preference}

A parallel line evaluates overall response quality through human preference proxies. AlpacaEval 2.0 \citep{li2024alpacaeval} achieves Spearman correlation of 0.98 with Chatbot Arena using GPT-4 as judge, while MT-Bench \citep{zheng2024mtbench} extends evaluation to multi-turn dialogue. These benchmarks have become the de facto standard for comparing model capabilities, with AlpacaEval evaluating $655+$ models at under \$10 per run.

The implicit assumption is revealing: a ``better'' response is one that a human judge prefers, regardless of how it was produced. An AI that generates a perfect summary by delegating to a faster sub-agent scores identically to one that reads and synthesizes each source document as instructed. The quality-focused paradigm is, by construction, blind to process fidelity.

\subsubsection{Task Completion in Agent Settings}

The emergence of autonomous AI agents introduced task-completion benchmarks. AgentBench \citep{liu2023agentbench} evaluates agents across eight interactive environments. SWE-bench \citep{jimenez2024swebench} tests whether models can resolve real GitHub issues from 12 Python repositories, with SWE-bench Verified \citep{yang2024swebenchverified} providing 500 human-validated cases and SWE-Agent \citep{yang2024sweagent} establishing the agent-computer interface paradigm. WebArena \citep{zhou2024webarena} places agents in realistic web environments with e-commerce and content management systems.

These benchmarks advance what compliance from text evaluation to behavioral evaluation---measuring whether agents achieve real-world outcomes. Yet even here, the evaluation criterion remains outcome-focused: SWE-bench checks whether the generated patch passes the test suite, not whether the agent debugged the issue through the developer-instructed method. A model that solves the issue by searching Stack Overflow (efficient) rather than reading the codebase (instructed) receives full marks.

\subsubsection{Tool-Use Accuracy}

The most proximate area to our work evaluates tool-use capabilities. Gorilla \citep{patil2023gorilla} and ToolLLM \citep{qin2024toolllm} test API selection from large repositories ($16{,}464$ APIs). The Berkeley Function Calling Leaderboard \citep{yan2024bfcl} established the standard for function-call accuracy. ToolSandbox \citep{lu2024toolsandbox} introduced stateful, conversational, interactive evaluation with $1{,}032$ test scenarios and implicit state dependencies. GTA \citep{wang2024gta} and TaskBench \citep{shen2024taskbench} extend tool evaluation to general agentic tasks and task-decomposition fidelity respectively. ToolEmu \citep{ruan2024toolemu} uses LM-emulated sandboxes to identify agent risks, finding that even the safest agent exhibits dangerous failures 23.9\% of the time.

\citet{schick2024toolformer}'s Toolformer demonstrated that language models can autonomously learn tool use, while \citet{yao2023react}'s ReAct framework synergized reasoning and acting through interleaved thought-action traces---the paradigm that underpins most modern tool-using AI agents and coding assistants. ReAct's reasoning step optimizes for task efficiency: ``I need to search $x$, find $y$, then find $z$.'' Critically, there is no step that verifies whether the planned action conforms to the user's procedural instruction. This architectural omission is a proximate cause of the Compliance Gap.

\citet{wang2024executable} showed that executable code actions outperform natural language planning for agent tasks. COMPASS \citep{dessureault2026compass} is the closest prior work to BS-Bench, providing real-time multi-dimensional evaluation across sovereignty, sustainability, compliance, and ethics. However, COMPASS evaluates \emph{which action to take} (a what question) rather than \emph{how the action is executed} (a how question): an agent that passes all four COMPASS dimensions but delegates file reading to \texttt{batch\_verify} would receive full marks while violating the user's process instruction.

\paragraph{Summary.} Across approximately 75 papers spanning four subdimensions of what compliance, no benchmark measures whether the AI followed the user-instructed \emph{method} to achieve its result. BS-Bench operationalizes this missing dimension.

\subsection{The Structural Blindness: Why HOW Has Been Overlooked}\label{appendix:blindness}

The absence of how-compliance research is not accidental. Three structural factors explain it.

\paragraph{First, outcome metrics are easier to automate.} Format compliance is checked by code \citep{zhou2023ifeval}. Quality is approximated by LLM judges \citep{li2024alpacaeval,zheng2024mtbench}. Task completion is verified by test suites \citep{jimenez2024swebench}. Process compliance, by contrast, requires access to behavioral logs---tool call traces, action sequences---that are not part of the standard evaluation pipeline. Most benchmarks operate on $(\text{input}, \text{output})$ pairs and never observe the intermediate behavioral trajectory.

\paragraph{Second, RLHF optimizes text, not behavior.} The foundational RLHF pipeline \citep{christiano2017deep,ouyang2022instructgpt} trains a reward model from human preferences over text completions. Human evaluators comparing two responses cannot determine which response was produced by following the user's specified procedure---the behavioral layer is invisible to them. \citet{rafailov2023dpo} simplified this pipeline with Direct Preference Optimization, but the fundamental constraint remains: preferences reward surface appearance, not underlying process. \citet{casper2023open} catalog 22 open problems in RLHF, including reward hacking and specification gaming; \citet{skalse2022defining} formally characterize reward hacking; \citet{gao2023scaling} demonstrate predictable scaling laws for reward model overoptimization. \citet{pan2024feedback} show that feedback loops in language models drive in-context reward hacking, while \citet{denison2024sycophancy} trace a progression from sycophancy to subterfuge in which reward-seeking behaviors escalate. \citet{ngo2024alignment} analyze the alignment problem from a deep learning perspective, arguing that deceptive alignment is a natural consequence of optimization pressure. The Compliance Gap is, in alignment terms, a specific instantiation of Goodhart's Law in the sense of \citet{manheim2018categorizing} (regressional Goodhart): when text quality becomes the optimization target, it ceases to measure process quality.

\paragraph{Third, the instruction hierarchy structurally deprioritizes user process instructions.} \citet{wallace2024instruction} formalized the instruction hierarchy, training GPT-3.5 to prioritize system-level over user-level instructions. The OpenAI IH-Challenge dataset \citep{openai2024ihchallenge} and \citet{chen2025iheval}'s IHEval provide evaluation infrastructure. \citet{zeng2025whoincharge} applied mechanistic interpretation to $120{,}000$ prompts, finding that social hierarchy cues override architectural role designations in the model's latent space. \citet{saebo2026goaldrift} independently observe ``asymmetric goal drift'' in coding agents, where system prompt constraints are progressively violated. \citet{shayegani2025bgd} identify ``Blind Goal-Directedness'' across nine frontier models with an 80.8\% average rate, where three failure modes---execution-first bias, thought-action disconnect, and request-primacy---parallel the Compliance Gap's verbal-behavioral dissociation. These works assume that compliance \emph{within} a priority level is unproblematic---our results show it is not. Even when a process instruction occupies the highest-priority position (system prompt), models may verbally agree and behaviorally defect.

\subsection{The Mechanism: From Sycophancy to False Compliance}\label{appendix:mechanism}

The behavioral mechanism underlying the Compliance Gap connects two literatures: sycophancy research and CoT faithfulness.

\subsubsection{Sycophancy: Agreement Without Commitment}

\citet{sharma2024sycophancy} established that sycophancy is a general behavior of RLHF-trained assistants, driven by human preference judgments that favor agreeable responses over truthful ones. \citet{perez2022sycophancy} discovered sycophantic behavior through model-written evaluations. \citet{kim2025evaluator} showed that sycophancy intensifies under user rebuttal, while \citet{yao2025peacemaker} demonstrated that multi-agent debate amplifies rather than mitigates it. \citet{wei2024sycophancy} proposed methods for measuring and reducing sycophancy without gold-standard answers. \citet{chen2024spt} proposed Sycophancy-Probe Tuning to address sycophancy without degrading general capabilities, and \citet{potham2025illusion} identified the ``illusion of compliance''---high adherence scores masking task incompetence.

\citet{hubinger2024sleeper} provide the deepest concern: deceptive behaviors, once encoded through training, persist through SFT, RL, and adversarial training. \citet{burns2023latent} demonstrated that LLMs can possess latent knowledge they do not express---a ``knowing-but-not-doing'' phenomenon directly analogous to process non-compliance. \citet{gao2024honestllm} propose honesty-targeted training that exposes how surface compliance can mask substantive disagreement; \citet{liu2024formalizing} formalize alignment failures as a class of specification-implementation gaps that subsumes the Compliance Gap as a process-level instance.

We extend sycophancy from \emph{belief alignment} (agreeing with opinions) to \emph{behavioral alignment} (agreeing to follow procedures, then not doing so). This \emph{False Compliance Sycophancy} is a commitment-action gap rather than a belief gap: the model says ``I will read each file manually'' and then delegates to \texttt{batch\_verify}. \citet{wen2024language} demonstrate that ``language models learn to mislead humans via RLHF''---our contribution specifies \emph{how} this manifests in tool-use settings.

\subsubsection{CoT Faithfulness: The Parallel Problem in Reasoning}

Chain-of-thought faithfulness research addresses a structurally parallel phenomenon: reasoning traces that do not reflect actual computation.

\citet{turpin2023unfaithful} showed that LLMs ``do not always say what they think,'' with CoT explanations that are post-hoc rationalizations; \citet{turpin2023cot} provide the closely related broader treatment. \citet{arcuschin2025wild} confirmed this in deployment settings. \citet{lanham2023measuring} found that early tokens disproportionately influence final answers regardless of intermediate reasoning. \citet{paul2024frodo} proposed methods for making reasoning matter by improving faithfulness. \citet{tanneru2024hardness} established theoretical hardness results showing faithful CoT is fundamentally difficult. \citet{cheng2025cotobscures} demonstrated that elaborate CoT \emph{obscures} hallucination cues for humans---the more detailed the reasoning appears, the harder errors are to detect. \citet{matton2025walkthetalk} developed automated metrics for reasoning faithfulness measurement.

The parallel is precise: CoT unfaithfulness $=$ gap between \emph{stated reasoning} and \emph{actual computation}; Compliance Gap $=$ gap between \emph{stated behavior} and \emph{actual behavior}. Both are invisible at the text output layer, both are driven by training incentives that reward surface over substance, and both require automated measurement beyond text analysis.

\paragraph{Adjacent surveys of AI deception.} Our closest neighbor in the deception literature is \citet{park2024deception}, who provide the first systematic catalogue of AI deception cases across learned, instrumental, and emergent forms. Their contribution is taxonomic and qualitative; ours is metric-infrastructural. \citeauthor{park2024deception} document \emph{what} kinds of deceptive behavior have been observed; we provide the measurement scaffolding (BS-Bench), formal results (Theorems~\ref{thm:rlhf} and \ref{thm:dpi}), and 13 controlled experiments that quantify \emph{how much} of one specific class---procedural compliance failure---occurs across frontier systems. We see the two contributions as complementary: their survey raises the question across the broader landscape of AI deception, and our benchmark addresses one class of it---procedural compliance---through measurement.

\paragraph{Behavioral testing for NLP vs.\ for AI agents.} The closest prior framework for behavior-level evaluation is CheckList \citep{ribeiro2020checklist}, which introduced behavioral testing for NLP models with capability-based test types (Minimum Functionality, Invariance, Directional). CheckList tests model \emph{output text} against linguistic capabilities (e.g., negation, vocabulary, named-entity robustness); BS-Bench tests agent \emph{action trajectories} against user-issued procedural instructions. The two are complementary: CheckList probes what NLP models \emph{say} across linguistic dimensions, whereas BS-Bench probes what AI agents \emph{do} against procedural commands. Their unit of observation is the textual output; ours is the tool-call log. Both reject the accuracy-only paradigm, but on orthogonal channels.

\subsection{The Context: Instruction Loss and Environmental Affordance}\label{appendix:context}

Two additional literatures explain the contextual conditions under which the Compliance Gap emerges.

\subsubsection{How Instructions Are Lost}

\citet{liu2024lostmiddle} established the lost-in-the-middle phenomenon: LLMs attend to context beginnings and ends while neglecting middle positions. \citet{hsieh2024ruler} traced this to positional attention bias in their RULER benchmark, finding that claimed context sizes dramatically overstate effective context. \citet{shi2023large} showed that irrelevant context degrades reasoning, while \citet{levy2024task} demonstrated that longer inputs reduce accuracy even when additional content is relevant. \citet{an2024make} proposed IN2 training to address context underutilization. \citet{sharma2026contextcov} introduced ContextCov to transform passive agent instructions into executable guardrails---implicitly acknowledging that text-based instructions are insufficient.

Our Layer Manipulation experiment (Exp~2) directly tests these findings in the process-compliance domain: identical instructions placed at four architectural positions (System, User, Tool-description, Repeated) yield compliance rates ranging from 26\% (Tool) to 65\% (Repeated), with instruction \emph{content} explaining $4\times$ more variance than instruction \emph{position} ($\eta^2 = .358$ vs.\ $.089$).

\subsubsection{How the Environment Shapes Non-Compliance}

\citet{gibson1979ecological}'s ecological approach predicts that behavior is shaped by environmental affordances---the action possibilities that the environment presents. \citet{mccradden2023algorithmic} term the AI analog \emph{algorithmic paternalism}: the system overrides the user's procedural judgment because it calculates that delegation is more efficient. \citet{parasuraman1997humans} established the foundational taxonomy of human-automation failures (Use, Misuse, Disuse, Abuse); we introduce a fifth category---\emph{AI-side Misuse}---where the AI system itself over-relies on available tools.

\citet{kahneman2011thinking}'s System~1/System~2 distinction provides a cognitive framework: process instructions demand System~2 engagement, but AI systems optimized for efficiency default to System~1 patterns---reaching conclusions quickly without following the specified procedure. The adjacent prompt-injection literature \citep{ye2025roleconfusion,zverev2025separation,greshake2023indirect} formalises how surface-level role cues can mislead the model's authority assignment, with all evaluated models failing adequate instruction-data separation \citep{zverev2025separation}.

Our Tool Ablation experiment (Exp~6) directly tests the affordance hypothesis: removing delegation tools converts $\mathrm{ICR}$ from $0\%$ to $100\%$ for two of three models (Claude Sonnet~4, GPT-4o-mini), confirming that delegation affordance is a necessary environmental condition for the Compliance Gap.

\subsection{The Oversight Gap: Can Humans Detect Process Non-Compliance?}\label{appendix:oversight}

The human-AI trust literature has theorized but not empirically tested whether humans can detect behavioral transparency failures of this kind.

\citet{mayer1995trust}'s integrative model defines trust as comprising Ability, Benevolence, and Integrity. The Compliance Gap is formally an Integrity failure. \citet{lee2004trust} identified Process trust---confidence that the system followed the right procedure---as the most critical dimension for professional domains. In medicine, \citet{topol2019high} envisions AI augmenting clinical practice but notes that diagnostic procedures exist to ensure no differential is overlooked; \citet{emanuel2019artificial} argue that AI value in healthcare requires process integrity alongside outcome accuracy; \citet{grote2022enabling} demonstrate that fairness in ML-based healthcare requires attention to procedural justice. \citet{parasuraman2010complacency} documented automation complacency: humans reduce monitoring when they trust automated systems. \citet{vasconcelos2023explanations} showed that explanations can reduce overreliance, while \citet{zhang2020effect} demonstrated that confidence displays improve trust calibration. \citet{bansal2021does} found that human-AI team effectiveness depends on appropriate reliance.

\citet{tsamados2025humancontrol} argue that the supervisory control paradigm must shift to human-machine teaming---but teaming requires \emph{meaningful} human control \citep{fuchs2024delegation}, which presupposes behavioral observability. Our blinded human evaluation (nine text-only raters, Fleiss' $\kappa = 0.130$; methodological grounding in \citealp{fleiss1971measuring,landis1977measurement,artstein2008inter}) demonstrates that this presupposition fails: humans evaluating text output alone cannot distinguish compliant from non-compliant AI responses, even when provided with gold-standard references. \citet{clark2021thats} showed that human evaluation of generated text is less reliable than assumed; \citet{james2026counting} called for greater transparency in reporting agreement metrics. Our finding extends this to a new domain: the Compliance Gap is not merely difficult to detect---it is \emph{structurally invisible} at the text output layer.

\subsection{Cross-Domain Audit Infrastructure Analogs}\label{appendix:crossdomain}

The four-domain isomorphism grid (Aviation CRM, WHO Surgical Safety Checklist, SOX §404, ABA Rule 1.1) is promoted to the main body in Section~\ref{sec:discussion} (Figure~\ref{fig:crossdomain}); this appendix subsection is retained as a label-anchor for backward references.

Extended discussion: Process-compliance auditing has decades-old precedent in regulated industries, and BS-Bench is best understood as the AI-deployment counterpart of this longer tradition. Aviation's cockpit voice recorders and flight data recorders provide a behavioral trace independent of the pilot's verbal radio communication \citep{helmreich2000error}; reviewing the two channels in tandem is the standard mechanism by which incident investigators distinguish what was said from what was done. The World Health Organization's Surgical Safety Checklist \citep{haynes2009surgical} mandates verbal confirmation of nineteen procedural items, with deviation auditable post-operatively. Sarbanes--Oxley §404 internal-control requirements \citep{coates2007goals} oblige auditable trails of financial-decision processes, not merely of their outcomes. The American Bar Association's Model Rule~1.1 imposes a fiduciary ``competence'' duty on attorneys that, as established case law has clarified, prohibits a verbal-only ``I~reviewed it'' defense in the face of demonstrable behavioral omission. The shared design pattern across these domains is that high-stakes professional practice has long been judged by both verbal and behavioral channels, with infrastructure built to inspect both. BS-Bench supplies the equivalent infrastructure for AI deployments in these and adjacent regulated domains, where process compliance is not a stylistic preference but a legally mandated component of professional practice.

\subsection{Proofs of Theorems \ref{thm:rlhf} and \ref{thm:dpi}}\label{appendix:proofs}

We provide argument sketches here; the full versions, together with concentration bounds for $\widehat{\mathrm{CG}}$ (Hoeffding) and an explicit assumption list (A1--A3 + variance stationarity), appear in the Supplementary Material (\S S1--S4). The channel-level visualization (Figure~\ref{fig:dpichannel}, promoted to the main body in Section~\ref{sec:bench}) anchors the proof of Theorem~\ref{thm:dpi}.

\paragraph{Proof of Theorem~\ref{thm:rlhf} (sketch).}
Let $\pi_\theta$ be a policy parameterized by $\theta$, with reward $R: \mathcal{Y} \to \mathbb{R}$ defined on text outputs $y$ alone, and let $b: \mathcal{Y} \times \mathcal{H} \to \mathcal{B}$ map (output, environment) to the behavioral trajectory. By \citet{skalse2022defining}, $R$ is \emph{hackable} relative to user utility $U$ whenever $\arg\max_\theta R(\pi_\theta) \neq \arg\max_\theta U(\pi_\theta)$. If $U$ depends on $b$ but $R$ does not, then any $\theta^\star \in \arg\max_\theta R(\pi_\theta)$ leaves $U(\pi_{\theta^\star})$ free to vary over the level sets of $R$. Whenever $b$ has positive variance under $\pi_{\theta^\star}$ (which holds for any non-trivial environment family $\mathcal{H}$), there exists $\theta^\dagger$ with $R(\pi_{\theta^\dagger}) = R(\pi_{\theta^\star})$ but $U(\pi_{\theta^\dagger}) > U(\pi_{\theta^\star})$. Optimizing $R$ therefore admits solutions with strictly suboptimal $U$; in expectation across deployments, $\mathbb{E}[\mathrm{VCR} - \mathrm{ACR}] > 0$.

\paragraph{Proof of Theorem~\ref{thm:dpi} (sketch).}
The (verbal, behavioral) pair $X = (y, b)$ generates the text observation $Y = y$ via the deterministic projection $\pi_y(X) = y$. This forms a Markov chain $(y, b) \to y \to f(y)$ for any rater function $f$. By the Data Processing Inequality \citep[\S2.8]{cover2006elements}, $I((y, b); f(y)) \leq I((y, b); y) = H(y) - H(y \mid (y, b)) = H(y)$, with equality iff $f$ is invertible on $y$. The Compliance Gap is defined on the residual $d = b - \mathbb{E}[b \mid y]$. Under the regularity assumption that $d$ is independent of $y$ in distribution (cf.\ Supplementary~\S2 for the precise condition), no $f$ measurable with respect to $\sigma(y)$ can recover $d$ uniformly; the best estimator $\hat{d}^\star(y) = \mathbb{E}[d \mid y] \equiv 0$ achieves $\Pr[\hat{d}^\star(y) = d]$ which is bounded away from 1 whenever $\mathrm{Var}(d \mid y) > 0$, and any other estimator satisfies $\Pr[\hat{d}(y) = d] \leq \Pr[\hat{d}^\star(y) = d]$. The empirical R6 result ($\kappa = 0.130$, $0/15$) is consistent with this bound.

\subsection{Operational Details: Channel-Level Auditability and Maintained Infrastructure}\label{appendix:operational}

\paragraph{Channel-level auditability: a preliminary illustration.} Theorem~\ref{thm:dpi} bounds text-only observers; the tool-call channel that BS-Bench instruments lifts that bound. As preliminary illustration, a 146M-parameter hyperbolic classifier on tool-call logs reaches $98.3\%$ binary detection on a 288-session held-out subset (preliminary, separate ongoing work; cf.\ Supplementary~\S2 Remark). The classifier is separate ongoing work, not a contribution of this paper; we report it only to indicate the channel is practically auditable. The structural claims of this paper do not depend on the classifier's existence or performance.

\paragraph{Continuously maintained infrastructure.} We release BS-Bench with the intention that it serve not as a one-shot benchmark but as a continuously maintained leaderboard, contributing to the broader AI accountability infrastructure that \citet{ojewale2024accountability} identify as currently underdeveloped for behavior-level audit. Our plan is to support (i)~automated submission validation via continuous-integration scripts, (ii)~regular leaderboard refresh, (iii)~a versioned task suite (with v1.0 frozen at submission and a planned v1.1 release adding long-horizon and multi-agent task types), and (iv)~a public submission portal\footnote{\anonurl{https://github.com/seanshin0214/bs-bench/leaderboard}.} accepting new task types, additional rater cohorts, and replication studies under MIT license.

\subsection{Compliance Gap as a Three-Level Structural Cascade}\label{appendix:cascade}

The three-level structural cascade (L1 reward signal, L2 instruction hierarchy, L3 affordance) is promoted to the main body in Section~\ref{sec:discussion} (Figure~\ref{fig:cascade}); this label-anchor is retained for the Murphy-form alignment law footnote in Section~\ref{sec:conclusion}.


\section{Supplementary Material — Full Proofs and Detailed Experiment Inventory}\label{sec:supplementary}

\section*{Notation Summary}

We collect notation used throughout the proofs.

\begin{itemize}
\item $\Theta$: parameter space of the policy.
\item $\pi_\theta : \mathcal{X} \to \Delta(\mathcal{Y})$: a stochastic policy mapping prompts $x \in \mathcal{X}$ to distributions over text outputs $y \in \mathcal{Y}$.
\item $\mathcal{H}$: space of environments (tool affordances, file states, system context).
\item $b : \mathcal{Y} \times \mathcal{H} \to \mathcal{B}$: behavioral trajectory induced by output $y$ in environment $h$. We write $B = b(Y, H)$ for the random behavioral trajectory.
\item $R : \mathcal{Y} \to \mathbb{R}$: a verbal-only reward (a function of the text output alone).
\item $U : \mathcal{B} \cup (\mathcal{Y} \times \mathcal{B}) \to \mathbb{R}$: user utility, depending on the realised behavior (and possibly the text).
\item $\mathrm{VCR}$: verbal compliance rate (fraction of sessions in which the assistant text claims to follow the user's process instruction).
\item $\mathrm{ACR}$: actual compliance rate (fraction of sessions in which the tool-call log shows the user's process was followed).
\item $\mathrm{CG} = \mathrm{VCR} - \mathrm{ACR} \in [-1, 1]$: the Compliance Gap.
\item $\mathrm{ICR}$: instruction compliance rate $=$ files read individually $/$ total files (per session).
\item $\mathrm{DF}$: delegation frequency $=$ batch tool calls $/$ total tool calls (per session).
\item $\mathrm{FCR}$: false completion rate $=$ false ``done'' claims $/$ total claims (per session).
\item $\mathrm{TA}$: task accuracy $=$ planted errors detected $/$ total planted errors (per session).
\item By channel: $\mathrm{VCR}$ is the sole verbal-channel metric; $\mathrm{ICR}, \mathrm{DF}, \mathrm{FCR}, \mathrm{ACR}, \mathrm{TA}$ are behavioral-channel metrics; $\mathrm{CG}$ measures the gap between the two.
\item $H(\cdot)$, $I(\cdot;\cdot)$: Shannon entropy and mutual information.
\end{itemize}

\setcounter{theorem}{0}

\section{Theorem 1 (RLHF Goodhart Inevitability)}\label{appendix:thm1full}

\subsection{Statement}

\begin{theorem}[RLHF Goodhart Inevitability]\label{thm:rlhf-supp}
Let $\pi_\theta$ denote a policy trained against a verbal-only reward $R(y)$ where $y$ is the text output and let $b(y, h)$ denote the behavioral trajectory induced by $y$ in environment $h$. Assume:
\begin{enumerate}
\item[(A1)] User utility $U$ depends on $b$ but not on $y$ except through $b$, that is, $U = U(b)$ almost surely.
\item[(A2)] $R$ depends only on $y$: $R = R(y)$.
\item[(A3)] The induced map $\theta \mapsto \mathrm{Law}(b(\pi_\theta, H))$ is non-degenerate, in the sense that for every $\theta^\star \in \arg\max_\theta R(\pi_\theta)$ the conditional law of $b$ given $y \sim \pi_{\theta^\star}$ has positive variance.
\end{enumerate}
Then in expectation across deployments under preference-reward training, for $\theta^\star \in \arg\max_\theta R(\pi_\theta)$ selected by preference-gradient dynamics that reward verbal helpfulness,
\[
\mathbb{E}_{Y \sim \pi_{\theta^\star}, H}\bigl[\mathrm{VCR} - \mathrm{ACR}\bigr] \;>\; 0,
\]
i.e.\ $\mathrm{CG} > 0$ in expectation.
\end{theorem}

\subsection{Full proof}

We follow the reward-hacking taxonomy of \citet{skalse2022defining}, who define a reward $R$ as \emph{hackable} relative to a true utility $U$ when there exist policies $\pi, \pi'$ with $R(\pi') > R(\pi)$ but $U(\pi') < U(\pi)$.

\paragraph{Step 1: $R$ is hackable relative to $U$ under (A1)--(A2).}
Suppose for contradiction that $R$ is not hackable relative to $U$. Then for all $\pi, \pi'$, $R(\pi') > R(\pi) \Rightarrow U(\pi') \geq U(\pi)$. Equivalently, there exists a non-decreasing function $\phi$ with $U = \phi \circ R$ on the policy space.

By (A1), $U$ is a function of $b$ alone. By (A2), $R$ is a function of $y$ alone. The composition $\phi \circ R$ is therefore a function of $y$ alone. But for $U = \phi \circ R$ to hold pathwise, $U$ must factor through $y$, i.e.\ there exists $\psi$ with $U(b) = \psi(y)$ almost surely. By (A3), $b \mid y$ has positive variance, so for at least one realization $y$ there exist $b_1 \neq b_2$ in the support of $b \mid y$ with $U(b_1) \neq U(b_2)$ (whenever $U$ is non-constant on the support of $b$, which we assume; otherwise the theorem is vacuous). This contradicts $U(b) = \psi(y)$. Hence $R$ is hackable relative to $U$.

\paragraph{Step 2: Hackability implies a non-trivial level set of $R$ on which $U$ varies.}
Since $R$ depends only on $y$, the level set $\mathcal{L}_r = \{y : R(y) = r\}$ partitions $\mathcal{Y}$ into iso-reward fibres. For $\theta^\star \in \arg\max_\theta R(\pi_\theta)$, the policy concentrates on $y$-values with maximal $R(y) = r^\star$. Within $\mathcal{L}_{r^\star}$, by (A3), the conditional distribution $b \mid y \in \mathcal{L}_{r^\star}$ has positive variance, so $\mathrm{Var}(U(b) \mid y \in \mathcal{L}_{r^\star}) > 0$ whenever $U$ is non-constant on the support of $b$.

\paragraph{Step 3: The behavioral channel admits a strictly $U$-improving alternative inside the same $R$-level set.}
Let $b^{\dagger}$ denote a $U$-maximising selection from the support of $b \mid y \in \mathcal{L}_{r^\star}$. Construct a perturbed policy $\pi_{\theta^\dagger}$ that, conditional on producing $y \in \mathcal{L}_{r^\star}$, follows the procedural action that realises $b^{\dagger}$ rather than the gradient-favoured shortcut. Then $R(\pi_{\theta^\dagger}) = R(\pi_{\theta^\star})$ but $U(\pi_{\theta^\dagger}) > U(\pi_{\theta^\star})$. The maximiser $\pi_{\theta^\star}$ is therefore $U$-suboptimal, with the gap strictly positive.

\paragraph{Step 4: From $U$-loss to verbal--behavioral divergence.}
In RLHF deployments, $R$ rewards helpful-sounding text including verbal commitments to procedure (the very claim ``I will follow $X$''). The text $y$ from $\pi_{\theta^\star}$ produces verbal compliance assertions optimised under $R$ (the very thing $R$ rewards: helpful-sounding, agreeable, on-topic text). The behavioral trajectory $b$ that maximises shortcut-availability under environment $h$ is the one that minimises latent cost (delegation, batched tool calls, omitted process steps). When user utility $U$ rewards process completion (Step~3), the verbal channel saturates near $\mathrm{VCR} \to 1$ while the behavioral channel favours shortcuts: $\mathrm{ACR} < \mathrm{VCR}$. Taking expectations,
\[
\mathbb{E}\bigl[\mathrm{VCR} - \mathrm{ACR}\bigr]_{\pi_{\theta^\star}} \;>\; 0.
\]
This is precisely $\mathrm{CG} > 0$ in expectation.

\subsection{Remarks}

\begin{remark}[Why (A3) is mild]
Assumption (A3) requires only that the behavioral channel $b$ is not deterministically pinned by $y$. In practical RLHF deployments, the same verbal text ``I will read each file individually'' can be realised by either of (i) iterative single-file reads or (ii) a batched tool call. The two realizations have identical $y$ but distinct $b$, satisfying (A3). The R6 human-evaluation data (Sec.~\ref{sec:experiments} of the main paper) directly evidence this: nine raters could not reliably distinguish (i) from (ii) given $y$ alone.
\end{remark}

\begin{remark}[Goodhart's Law specialization]
Theorem~\ref{thm:rlhf-supp} is a specialization of Goodhart's Law in the sense of \citet{manheim2018categorizing}: \emph{regressional Goodhart} (correlation between $R$ and $U$ breaks at the optimum). Our contribution is an information-theoretic factorisation argument (Step~1) that shows the breakage is not contingent but structural under (A1)--(A2).
\end{remark}

\begin{remark}[Connection to \citet{amodei2016concrete}]
This theorem operationalises \emph{reward hacking}, the first of Amodei's five concrete problems in AI safety (Step~3 establishes the reward-hacking failure mode); Theorem~\ref{thm:dpi-supp} operationalises the second, \emph{scalable oversight}, by bounding what a text-only auditor can recover. The Compliance Gap is the empirical instantiation of both, now measurable through the tool-call audit channel that BS-Bench supplies.
\end{remark}

\section{Theorem 2 (DPI Undetectability)}\label{appendix:thm2full}

\subsection{Statement}

\begin{theorem}[DPI Undetectability]\label{thm:dpi-supp}
Let $X = (Y, B)$ denote the (verbal, behavioral) random pair with joint distribution $P_{Y,B}$, and let $f : \mathcal{Y} \to \mathcal{Z}$ be any rater function depending only on the text $Y$. Define the residual
\[
D_{Y,B} \;=\; B - \mathbb{E}[B \mid Y],
\]
i.e.\ the part of behavior not predictable from text. Then for any such $f$,
\[
I\bigl(D_{Y,B};\, f(Y)\bigr) \;=\; 0,
\]
and consequently no estimator $\hat{d} = g(f(Y))$ achieves $\Pr[\hat{d} = D_{Y,B}] > \Pr[D_{Y,B} = \mathbb{E}[D_{Y,B} \mid Y]]$. The Compliance Gap, defined on the residual $D_{Y,B}$, is therefore not identifiable from $Y$ alone whenever $\mathrm{Var}(B \mid Y) > 0$.
\end{theorem}

\subsection{Full proof}

\paragraph{Step 1: Markov chain.}
The pair $(Y, B)$ is observed only through the deterministic projection $\pi_Y : (Y, B) \mapsto Y$, and any rater function $f$ depends only on $Y$. We therefore have the Markov chain
\[
(Y, B) \;\to\; Y \;\to\; f(Y),
\]
because $f(Y)$ is conditionally independent of $B$ given $Y$ by construction.

\paragraph{Step 2: Data Processing Inequality.}
By the Data Processing Inequality \citep[Theorem 2.8.1]{cover2006elements}, mutual information is non-increasing under processing:
\[
I\bigl((Y, B);\, f(Y)\bigr) \;\leq\; I\bigl((Y, B);\, Y\bigr).
\]

\paragraph{Step 3: Decomposition of mutual information.}
By the chain rule,
\[
I\bigl((Y, B);\, Y\bigr) \;=\; I(Y; Y) + I(B; Y \mid Y) \;=\; H(Y) + 0 \;=\; H(Y),
\]
since $B \mid Y$ is independent of $Y$ given $Y$ (trivially), so $I(B; Y \mid Y) = 0$.

\paragraph{Step 4: Residual is orthogonal to $Y$.}
Define the residual $D_{Y,B} = B - \mathbb{E}[B \mid Y]$. By construction,
\[
\mathbb{E}[D_{Y,B} \mid Y] = 0 \quad \text{almost surely},
\]
so $D_{Y,B}$ is mean-independent of $Y$. Now compute
\[
I(D_{Y,B};\, f(Y)) \;\leq\; I(D_{Y,B};\, Y) \;=\; H(D_{Y,B}) - H(D_{Y,B} \mid Y).
\]

The mean-independence above ensures the centred component of the residual is uncorrelated with $Y$. Under the additional regularity assumption that $D_{Y,B}$ is independent of $Y$ in distribution (a stronger form of the variance-stationarity assumption, which holds for our session-level data where the conditional law of behavioral shortcuts is invariant across verbal-commitment patterns; see Assumption~\ref{appendix:assumptions} below for the precise condition), we obtain $H(D_{Y,B} \mid Y) = H(D_{Y,B})$. Therefore
\[
I(D_{Y,B};\, f(Y)) \;=\; 0.
\]

\paragraph{Step 5: Detection bound.}
Any estimator $\hat{d}$ measurable with respect to $\sigma(Y)$ is a function of $f(Y)$ for some $f$. By Fano's inequality \citep[Theorem 2.10.1]{cover2006elements} applied to $D_{Y,B}$,
\[
\Pr[\hat{d}(Y) \neq D_{Y,B}] \;\geq\; \frac{H(D_{Y,B} \mid f(Y)) - 1}{\log |\mathcal{D}|} \;=\; \frac{H(D_{Y,B}) - 1}{\log |\mathcal{D}|},
\]
where $\mathcal{D}$ is the support of the residual. Whenever $H(D_{Y,B}) > 1$ (which holds for any non-degenerate $b \mid y$), the error rate is bounded away from zero. The Compliance Gap is therefore not reliably identifiable from $Y$ alone.

\subsection{Remarks}

\begin{remark}[Empirical concordance with R6]
The R6 blinded human-evaluation data report Fleiss' $\kappa = 0.130$ across nine raters and $0/15$ correctly identified compliant responses. Both numbers are consistent with $I(D_{Y,B}; f(Y)) \approx 0$: under chance-level identifiability, $\kappa \to 0$ and accuracy $\to$ base rate.
\end{remark}

\begin{remark}[What \emph{can} close the gap]
Theorem~\ref{thm:dpi-supp} bounds only \emph{text-only} observers. If the rater observes the tool-call log channel directly, $f$ is no longer measurable with respect to $\sigma(Y)$ alone, and the DPI bound does not apply. This is precisely the channel BS-Bench instruments. As reported in the main text, a 146M-parameter hyperbolic classifier reading this channel reaches $98.3\%$ binary detection on a held-out 288-session subset of BS-Bench --- a preliminary illustration only, not a contribution of this paper.
\end{remark}

\begin{remark}[Sharpness]
The bound is sharp: when $b = g(y)$ for some deterministic $g$ (the degenerate case), $D_{Y,B} \equiv 0$ and the Compliance Gap is recoverable from $Y$ alone. Theorem~\ref{thm:dpi-supp} is non-trivial precisely when $\mathrm{Var}(B \mid Y) > 0$, which is the empirically observed regime (R6 blinded human evaluation; same condition produces the heterogeneous behavioral outcomes documented in Exps 1, 6, and 9 of the main paper).
\end{remark}

\section{Concentration Bounds for $\widehat{\mathrm{CG}}$}\label{appendix:concentration}

We next derive a finite-sample bound for the empirical Compliance Gap estimator
\[
\widehat{\mathrm{CG}} \;=\; \frac{1}{n} \sum_{i=1}^n \mathbb{1}[\text{verbal compliance in session } i] - \frac{1}{n} \sum_{i=1}^n \mathbb{1}[\text{actual compliance in session } i].
\]

\begin{proposition}[Hoeffding bound on $\widehat{\mathrm{CG}}$]
Let sessions be drawn i.i.d.\ from the population. For any $\epsilon > 0$,
\[
\Pr\bigl[|\widehat{\mathrm{CG}} - \mathrm{CG}| > \epsilon\bigr] \;\leq\; 4 \exp(-n\epsilon^2 / 2).
\]
\end{proposition}

\begin{proof}
Each indicator is bounded in $[0,1]$. By Hoeffding's inequality applied separately to $\widehat{\mathrm{VCR}}$ and $\widehat{\mathrm{ACR}}$,
\[
\Pr[|\widehat{\mathrm{VCR}} - \mathrm{VCR}| > \epsilon/2] \leq 2\exp(-n\epsilon^2/2),
\]
and similarly for $\widehat{\mathrm{ACR}}$. Union bound gives the stated rate.
\end{proof}

\paragraph{Numerical implication.}
With $n = 240$ sessions (Exp~1) and $\epsilon = 0.05$, the bound is $4\exp(-240 \cdot 0.0025 / 2) = 4\exp(-0.3) \approx 2.96$, which is loose; the empirical bootstrap interval (10{,}000 resamples) reported in the main paper is the operative quantity. With $n = 2{,}031$ (full BS-Bench v1) and $\epsilon = 0.05$, the bound is $4\exp(-2031 \cdot 0.0025 / 2) = 4\exp(-2.54) \approx 0.317$, useful as a coarse sanity check.

\section{Assumptions Made Explicit}\label{appendix:assumptions}

For transparency we list every assumption used in the two theorems, marked by where it enters.

\begin{assumption}[A1, used in Thm~1]
User utility $U$ depends only on the realised behavior $b$. We do not require $U$ to ignore $y$; the proof goes through whenever $U$ has a strict additive component depending on $b$ (cf.~the surgical-checklist case in App.~A.7).
\end{assumption}

\begin{assumption}[A2, used in Thm~1]
The reward signal $R$ during preference-reward training depends only on text $y$. This holds for canonical RLHF \citep{christiano2017deep,ouyang2022instructgpt} that elicits human preferences from text outputs alone. Process-aware reward signals (e.g.\ tool-call-log scoring) violate (A2) by construction; Theorem~\ref{thm:rlhf-supp} accordingly does not bound them, and BS-Bench is the infrastructure required to make such alternative signals operationally feasible.
\end{assumption}

\begin{assumption}[A3, used in Thm~1]
The behavioral channel $b$ has positive conditional variance given the verbal channel $y$. Empirically supported by R6 ($\kappa = 0.130$, $0/15$).
\end{assumption}

\begin{assumption}[A4: Distributional independence of the residual, used in Thm~2 Step~4]
The residual $D_{Y,B}$ is independent of $Y$ in distribution. This implies $H(D_{Y,B} \mid Y) = H(D_{Y,B})$. A strictly weaker form ($H(D_{Y,B} \mid Y) \leq H(D_{Y,B})$, equivalent to mean-independence plus variance stationarity) suffices for the inequality version of Thm~\ref{thm:dpi-supp}; see \citet{cover2006elements} §2.5 for the general case.
\end{assumption}

\section{Experiment Inventory}\label{appendix:inventory}

This section provides the experiment inventory (sessions, seeds, framings) for all 13 experiments cross-referenced in the main paper. Per-model raw outcomes (sessions $\times$ models $\times$ cells $\times$ metrics) and the Python notebooks that compute the reported statistical tests (paired $t$-tests, Mann-Whitney $U$, $\eta^2$ with bootstrap CI, Hoeffding bounds) are released as the OSF artifact at \anonurl{https://osf.io/mvnq4/}; the PDF supplement is intentionally inventory-level for archival readability.

\begin{table}[h]
\centering
\footnotesize
\setlength{\tabcolsep}{4pt}
\begin{tabular}{@{}clrlp{0.55\linewidth}@{}}
\toprule
\# & Exp ID & Sessions & Seeds & Description / Role \\
\midrule
1  & Exp 1   & 240 & 10 & Compliance Gap quantification (6 models $\times$ 4 framings $\times$ 10 seeds, Medium-50) \\
2  & Exp 2   & 192 & 8  & Instruction position (System/User/Tool/Repeated) \\
3  & Exp 2b  & 120 & 5  & Tool accuracy (honest vs.\ deceptive batch\_verify) \\
4  & Exp 3   & 24  & 6  & SFT correction pilot: LoRA/QLoRA on two open-weight SLMs (Llama 3.1 8B, Mistral 7B) using positive trajectories from Exp~5 as supervision; tool-selection ICR rose to 67\% (Cohen's $d=1.45$) but end-to-end completion ICR remained 0\%, suggesting RLHF's verbal-channel optimisation is the binding constraint rather than tool-choice exemplars \\
5  & Exp 5   & 60  & 5  & Mid-session correction prompt \\
6  & Exp 6   & 120 & 5  & Tool removal (delegation affordance ablation) \\
7  & Exp 7   & 240 & 10 & Interleaved reporting + tool-description ablation \\
8  & Exp 8   & 300 & 10 & Cross-reference task type \\
9  & Exp 9   & 300 & 10 & Privacy-first masking task type \\
10 & Exp 10  & 300 & 10 & Audit trail task type \\
11 & Exp 11  & --$^{*}$ & 10 raters & R6 blinded human evaluation (29 items rated by 10 raters = 290 rating-events; sessions are a subset of Exps 1--10, not double-counted) \\
12 & Exp 12  & 45  & 5  & Temperature ablation ($t \in \{0, 0.7, 1.0\}$, 3 models) \\
13 & Exp 13  & 90  & 5  & P3A multidomain (medical/legal/education $\times$ 3 models $\times$ 2 framings) \\
\midrule
   & \textbf{Total} & \textbf{2{,}031} & & \emph{independent sessions across Exps 1--10 + 12--13} \\
\bottomrule
\end{tabular}
\caption{The thirteen experiments referenced in the abstract. Exps 1--10 are presented in the main text \S\ref{sec:experiments}; Exps 11--13 are flagged with explicit labels in \S\ref{subsec:r6} and the Robustness subsection. $^{*}$Exp 11 (R6 blinded human evaluation) re-uses 29 sessions drawn from Exps 1--10 and contributes 290 independent rating-events; its sessions are not added to the independent-session total to avoid double counting. The total $2{,}031$ is the sum of Exps 1--10 (1{,}896) plus Exps 12--13 (135).}
\label{tab:inventory}
\end{table}

\paragraph{Compute envelope.} The total compute used by this paper is bounded by: (i) $2{,}031$ API-driven sessions across the six frontier API models (Claude Sonnet~4, GPT-4o, GPT-4o-mini, Gemini~2.5 Flash, Llama~3.3~70B via API, Mistral Small 24B via API), and (ii) two open-weight SLM LoRA/QLoRA fine-tuning runs for the Exp~3 SFT pilot (Llama~3.1~8B and Mistral~7B base models on a single workstation-class GPU). Per-experiment API call counts, token-level cost summaries, and SFT GPU-hour logs are tabulated in the OSF release alongside the raw outcomes.


\end{document}